\documentclass[]{research}

\usepackage{amsmath}
\usepackage{amsfonts}
\usepackage{hyperref}
\usepackage{url}
\usepackage{booktabs}
\usepackage{colortbl}
\usepackage{multirow}
\usepackage{pgfplots}
\usepackage{graphicx}
\usepackage{float}
\usepackage{makecell}
\usepackage{algorithm}
\usepackage{algorithmic}
\usepackage{xcolor}
\renewcommand{\algorithmiccomment}[1]{\hfill \textcolor{gray}{$\triangleright$ \textit{#1}}}
\usepackage{array}
\pgfplotsset{compat=1.18}
\usepackage{arydshln}
\usepackage{wrapfig}
\usepackage{amsthm}
\usepackage{fontawesome5}
\usepackage{tikz}
\usetikzlibrary{positioning, arrows.meta, calc, decorations.pathreplacing, fit, backgrounds, shapes.geometric, shapes.misc}

\newtheorem{proposition}{Proposition}
\usepackage{pifont}
\usepackage{capt-of}

\makeatletter
\def\adl@drawiv#1#2#3{%
        \hskip.5\tabcolsep
        \xleaders#3{#2.5\@tempdimb #1{1}#2.5\@tempdimb}%
                #2\z@ plus1fil minus1fil\relax
        \hskip.5\tabcolsep}
\newcommand{\cdashlinelr}[1]{%
  \noalign{\vskip\aboverulesep
           \global\let\@dashdrawstore\adl@draw
           \global\let\adl@draw\adl@drawiv}
  \cdashline{#1}
  \noalign{\global\let\adl@draw\@dashdrawstore
           \vskip\belowrulesep}}
\makeatother

\newcommand{\Method}{\text{Matrix-to-Matrix RNN}}
\newcommand{\method}{M\ensuremath{\mathrm{^2}}RNN}
\newcommand{\cmark}{\textcolor{green!60!black}{\ding{51}}}
\newcommand{\xmark}{\textcolor{red}{\ding{55}}}
\newcommand{\tablewidth}{\textwidth}
\newcommand{\tablepostskip}{}
\newcommand{\relatedref}{Section \ref{section:related}}
\newcommand{\efficiencyref}{Section \ref{section:throughput}}
\newcommand{\forgetgateref}{Section \ref{sec:forget-gate}}
\newcommand{\statesizeref}{Section \ref{section:state-size}}
\newcommand{\mamba}{\text{Mamba-2}}
\newcommand{\gdn}{\text{Gated DeltaNet}}

\newcommand{\bb}[1]{\textbf{#1}}

\newcommand{\blue}[1]{\textcolor{blue}{#1}}

\newcommand{\tc}{\textsf{TC$^0$}}
\newcommand{\nc}{\textsf{NC$^1$}}

\definecolor{secblue}{HTML}{DCEAF7}
\definecolor{secgreen}{HTML}{DFF0D8}
\definecolor{secorange}{HTML}{FDF2E0}
\definecolor{secpurple}{HTML}{EDE4F3}
\definecolor{myblue}{HTML}{EEF4FF}

\title{\method: Non-Linear RNNs with Matrix-Valued States for Scalable Language Modeling}
\author[12]{Mayank Mishra}
\author[3\dagger]{Shawn Tan}
\author[1]{Ion Stoica}
\author[1*]{Joseph E. Gonzalez}
\author[45*]{Tri Dao}

\affiliation[1]{UC Berkeley}
\affiliation[2]{Learning Machine}
\affiliation[3]{Google DeepMind}
\affiliation[4]{Princeton University}
\affiliation[5]{Together AI}

\contribution[*]{Equal advising}
\contribution[\dagger]{Work done at MIT-IBM Watson Lab}

\abstract{Transformers are highly parallel but are limited to computations in the \tc~complexity class, excluding tasks such as entity tracking and code execution that provably require greater expressive power. Motivated by this limitation, we revisit non-linear Recurrent Neural Networks (RNNs) for language modeling and introduce \Method~(\method): an architecture with matrix-valued hidden states and expressive non-linear state transitions. We demonstrate that the language modeling performance of non-linear RNNs is limited by their state size, and show how the state size expansion mechanism enables efficient use of tensor cores. Empirically, \method~achieves perfect state tracking generalization at sequence lengths not seen during training. These benefits also translate to large-scale language modeling. In hybrid settings that interleave recurrent layers with attention, Hybrid \method~outperforms equivalent \gdn~hybrids by $0.4$--$0.5$ perplexity points on a 7B MoE model, while using $3\times$ smaller state sizes for the recurrent layers. Notably, replacing even a single recurrent layer with \method~in an existing hybrid architecture yields accuracy gains comparable to Hybrid \method~with minimal impact on training throughput. Further, the Hybrid \gdn~models with a single \method~layer also achieve superior long-context generalization, outperforming state-of-the-art hybrid linear attention architectures by up to $8$ points on LongBench. Together, these results establish non-linear RNN layers as a compelling building block for efficient and scalable language models.
}

\protected\def\correspondenceemail{\href{mailto:mayank_mishra@berkeley.edu}{\texttt{mayank\_mishra@berkeley.edu}}}
\correspondence{Mayank Mishra @ \correspondenceemail}

\metadata[{\faGithub}~Training Code]{\url{https://github.com/open-lm-engine/lm-engine}}
\metadata[{\hspace{-0.1em}\raisebox{-0.2em}{\includegraphics[height=1.2em]{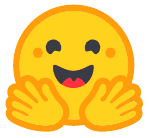}}}~Models]{\url{https://huggingface.co/collections/open-lm-engine/m2rnn}}

\begin{document}

\maketitle

\section{Introduction}
Foundation models trained on massive datasets and adapted to downstream tasks have become the dominant paradigm in modern machine learning \cite{gpt3, palm, llama3}. At their core, these models rely on sequence-based architectures capable of processing diverse modalities including text, code, images, video, and speech \cite{gpt3, palm, llama3, mishra2024granite, dosovitskiy2020image, lugosch2019speech}. The prevailing architecture is the decoder-only Transformer, composed of alternating multi-layer perceptron (MLP) and causal attention layers \cite{attention}. Much of attention's success stems from its amenability to parallelization across the sequence length, the availability of hardware-efficient implementations \cite{dao2022flashattention, dao2023flashattention, shah2024flashattention, zadouri2026flashattention4algorithmkernelpipelining}, and its effective in-context retrieval capabilities. Attention-based Transformers have also demonstrated remarkable scalability, powering models ranging from hundreds of billions \cite{gpt3, palm, workshop2022bloom, megatron-turing, liu2024deepseek, yang2025qwen3} to trillions of parameters \cite{team2025kimi, meta2024llama4}.

However, attention's quadratic time complexity during training and linearly growing memory requirements during inference have motivated the development of more efficient alternatives \cite{tay2022efficienttransformerssurvey}. State Space Models (SSMs) and linear attention \cite{gu2021efficiently, gu2021combining, mamba1, mamba2, gla, deltanet, gdn, rwkv} have emerged as promising replacements. \citet{katharopoulos2020transformers} propose computing linear attention via a dot product of kernel features, \citet{peng2021random} approximate attention using random feature methods, \citet{deltarule} introduce the delta update rule, showing significant improvements in associative recall, and \citet{mamba1, mamba2} propose Mamba with an efficient training algorithm that scales linearly with sequence length.

Linear attention admits both a recurrent and a parallel formulation. The constant-state size for linear attention enables memory-efficient and fast autoregressive inference (using the recurrent form) while the standard parallel form retains quadratic complexity during training. To bridge this gap, a chunkwise parallel formulation can be employed \cite{hua2022transformer, gla}, preserving linear time complexity while batching computations within chunks to leverage matrix multiplication units (tensor cores on NVIDIA GPUs \cite{nvidia_ptx}). SSMs like Mamba-1 \cite{mamba1} can alternatively be accelerated via the parallel scan algorithm \cite{parallelscan, BlellochTR90}.

Although linear RNNs have attracted considerable attention due to their efficient inference and training, they present several notable limitations:
\begin{enumerate}
    \item \textbf{Limited state tracking}: Linear RNNs with input-independent or diagonal transition matrices, such as \mamba~and \gdn, are provably less expressive than non-linear RNNs, particularly for hard state-tracking tasks such as code evaluation, entity tracking, and permutation composition ($S_5$ group permutation task); they are no more expressive than Transformers \cite{merrill2024illusion}. Linear RNNs with input-dependent, non-diagonal transitions can be more expressive \cite{grazzi2024unlocking, siems2025deltaproduct, terzic2025structured}, but training them stably requires constraining the spectral norm of the transition at every step.
    \item \textbf{Poor in-context retrieval performance}: Linear RNNs exhibit weak performance in in-context retrieval tasks \cite{gla,deltanet,gdn}. Because their recurrent state is updated via a fixed-rank outer product, the state can be overwritten when the number of key-value associations exceeds its capacity, leading to degraded in-context retrieval performance. This manifests as poor performance on needle-in-a-haystack benchmarks and long-context recall tasks that require retrieving specific information from distant context.
\end{enumerate}

A common strategy is to interleave quadratic attention layers with sub-quadratic recurrent layers in order to achieve a balance between inference efficiency and downstream task performance in state-of-the-art production models. This alleviates the poor language modeling and long-context performance of linear RNNs while retaining their efficiency advantages for long-context training and inference \cite{blakeman2025nvidia, lieber2024jamba, ren2024samba}.

While non-linear RNNs have been shown to excel at hard state tracking problems \cite{merrill2024illusion}, they face several practical challenges:
\begin{enumerate}
    \item \textbf{Poor language modeling performance}: Current non-linear RNNs significantly underperform linear RNNs in language modeling. We show that large state sizes are critical for strong language modeling performance, and that the inferior results of non-linear RNNs are primarily due to their smaller state sizes rather than their non-linearity.
    \item \textbf{Poor long-context retrieval performance}: We further show that the smaller state sizes significantly hurt the retrieval and long-context performance. The vector-valued non-linear RNNs in our experiments consistently underperform \mamba~and \gdn~by $\approx 20$ points in real-world retrieval tasks (Tables \ref{table:410m-real-retrieval}, \ref{table:7b-real-retrieval}).
    \item \textbf{Training inefficiency}: Unlike linear RNNs, non-linear RNNs cannot be parallelized across the sequence length, making training significantly more expensive. This is compounded by poor hardware utilization: na\"{\i}ve recurrent GEMM execution prevents reuse of on-chip tiles across time steps, forcing costly HBM I/O and synchronization. Block-wise variants such as FlashRNN \cite{flashrnn} keep the recurrence on-chip but map it onto tensor cores by batching sequences, which requires padding when the per-GPU micro-batch is small.
\end{enumerate}

We propose \Method~(\method), a non-linear RNN architecture that, when combined with hybrid linear RNNs, results in significant performance improvements in both language modeling and long-context performance. Specifically, \method~layers address the following key limitations of linear RNNs and non-linear RNNs:
\begin{enumerate}
    \item \textbf{Improved state tracking}: \Method~(\method) layer achieves perfect state tracking on $S_3$, significantly improving state tracking performance over diagonal linear RNNs similar to non-linear RNNs and GRUs \cite{gru}. We further demonstrate that \method~layers can express all computations that can be performed by a non-linear RNN.

    \item \textbf{Improved language modeling, in-context retrieval and long-context performance}: We demonstrate the importance of large state sizes for recurrent models and that it is a key driver behind the success of linear RNNs such as \mamba~\cite{mamba2} and \gdn~\cite{gdn}. We show that the outer product state expansion mechanism is directly applicable to non-linear RNNs. Similar to LSTMs \cite{lstm} and GRUs \cite{gru}, we use a forget gate to prevent gradient degradation across time steps \cite{lstmgruanalysis}, but unlike LSTMs and GRUs, our forget gate is independent of the recurrent state, enabling parallel computation. Hybrid \method~models outperform Hybrid \mamba~and Hybrid \gdn~on language modeling benchmarks. The outer product state expansion mechanism also improves in-context retrieval and long-context performance which are sensitive to the state size. Further, we demonstrate that combining \method~layers with Hybrid \mamba~and Hybrid \gdn~enhances long-context performance by up to 8 points at both 410M dense and 7B MoE scales (Tables \ref{table:410m-longbench}, \ref{table:7b-longbench}).

    \item \textbf{Improved hardware utilization}: The outer product based state expansion mechanism maps each recurrent step to a dense tensor-core GEMM whose shape is independent of the batch size, avoiding the batch padding that FlashRNN \cite{flashrnn} needs at small micro-batch sizes, at the cost of more FLOPs per step than vector-valued RNNs (Section \ref{sec:hardware-utilization}). Since \method~layers are computationally expensive and non-parallelizable across time steps, we further explore using \method~layers sparingly. We demonstrate with experiments that using \method~layers sparingly (even just one layer) results in substantial improvements in all benchmarks.
\end{enumerate}

Beyond efficient single-device execution, training models at scale requires careful attention to systems design. We describe the forward and backward kernels for \method~layers implemented in Triton \cite{triton}, and present two strategies for applying tensor parallelism (TP) to \method~layers: a topology-aware approach that requires no additional communication, and a topology-independent approach that preserves parameter count across TP configurations at the cost of extra synchronization (Section \ref{sec:tensor-parallel}). We discuss related work, including xLSTM \cite{beck2024xlstm}, FlashRNN \cite{flashrnn}, Griffin \cite{de2024griffinmixinggatedlinear} and Titans \cite{behrouz2024titans} in \relatedref.

\section{Background}
We use bold letters ($\bb{A}$) to represent matrices and small letters ($a$) to represent column vectors in the following sections.

\subsection{State Tracking}
State tracking refers to determining the state of a system by observing a sequence of updates applied to it. The complexity of state tracking tasks is characterized by circuit complexity classes:

\tc: \tc~denotes the class of languages decidable by constant-depth threshold boolean circuits of polynomial size. Simple state tracking problems such as parity (equivalent to the permutation group $S_2$) fall within this class. \citet{merrill-etal-2022-saturated} show that Transformers \cite{attention} without chain-of-thought lie in \tc, making them capable of solving only these simpler tasks (Section \ref{section:related}).

\nc: \nc~denotes the class of languages decidable by logarithmic-depth boolean circuits of polynomial size\footnote{Depth is measured in the input size.}. Many practical state tracking tasks such as tracking chess moves or evaluating programs reduce to problems at least as hard as the $S_5$ word problem, where $S_k$ denotes the symmetric group on $k$ elements (i.e., all permutations of $k$ objects). Since $S_5$ is \nc-complete, these tasks lie strictly outside the expressivity class of both Transformers \cite{merrill-etal-2022-saturated} and linear SSMs with input-independent or diagonal transition matrices \cite{merrill2024illusion}, motivating architectures with greater computational expressivity.

\subsection{Linear RNNs}

Both linear attention mechanisms and SSMs admit recurrent and parallel formulations, leading the literature to refer to them collectively as \emph{recurrent neural networks} (RNNs). However, \citet{merrill2024illusion} demonstrate a significant expressivity gap between these \emph{linear} RNNs and \emph{non-linear} RNNs. Non-linear RNNs, recurrent networks that apply a non-linear activation at each step of the recurrence, have traditionally been employed for sequence modeling \cite{elman1990finding}, particularly for tasks requiring state tracking across time steps. While the inclusion of a non-linearity yields a much richer class of representable functions, it precludes the use of algorithms like parallel scan \cite{parallelscan,  BlellochTR90} for efficient computation \cite{merrill2024illusion}. Nevertheless, the state-tracking capabilities enabled by non-linear RNNs can make their incurred cost a reasonable tradeoff.

Many linear attention and SSM models take the following general form where Equation \ref{equation:linear-attention-transition} is used to compute the recurrent state at the next time step and Equation \ref{equation:linear-attention-output-readout} is used to compute the final output:
\begin{align}
    \bb{H}_t &= \bb{H}_{t-1} \bb{A}(x_t) + \bb{B}(x_t) \label{equation:linear-attention-transition} \\
    y_t &= \bb{H}_{t}^\top c(x_t) + d(x_t) \label{equation:linear-attention-output-readout}
\end{align}
where $\bb{H}_t \in \mathbb{R}^{K \times V}$ denotes the recurrent state at time step $t$, $\bb{A}(\bb{x}_t) \in \mathbb{R}^{V \times V}$ is the transition matrix, and $\bb{B}(x_t) \in \mathbb{R}^{K \times V}$ is a function of the input $x_t \in \mathbb{R}^d$. $c(x_t) \in \mathbb{R}^{K \times 1}$ is the vector used to compute the output and $d(x_t) \in \mathbb{R}^{V \times 1}$ is the residual stream. $K$ and $V$ represent the key and value head dimensions respectively. Note that the recurrent state $\bb{H}_t$ is a matrix rather than a vector. \citet{katharopoulos2020transformers} replace the similarity score in attention ($\exp(q_i^\top k_j)$) \cite{attention} with a kernel $\kappa(q_i, k_j) = \left< \phi(q_i), \phi(k_j) \right>$ to compute the similarity scores between queries ($q_i$) and keys ($k_j$) where $\phi$ represents the feature map for the kernel function $\kappa$. This leads to the following simplified recurrence:
\begin{align}
    \bb{H}_t &= \bb{H}_{t-1} + \phi(k_t) {v_t}^\top \\
    z_t &= z_{t-1} + \phi(k_t) \\
    y_t &= \frac{\bb{H}_t^\top \phi(q_t)}{z_t^\top \phi(q_t)} \label{eqn:linearattn-full}
\end{align}
Prior works found that removing the normalization term in the denominator ($z_t^\top \phi(q_t)$) in Equation \ref{eqn:linearattn-full} improves training stability \cite{sun2023retentive} for linear attention and hence we omit this in going forward. The recurrence then becomes:
\begin{align}
    \bb{H}_t &= \bb{H}_{t-1} + \phi(k_t) {v_t}^\top \\
    y_t &= \bb{H}_t^\top \phi(q_t) \label{eqn:linearattn}
\end{align}

A comprehensive summary of recurrent forms for various linear attention models can be found in \citet{deltanet} (Table 2).

\subsection{Limitations of Linear RNNs}
\subsubsection{Limited State Tracking}
\citet{merrill2024illusion} demonstrate that hard state tracking tasks like the permutation group $S_5$ cannot be solved when the transition matrix for a linear RNN is input-independent ($\bb{A}(x_t) = \bb{A}$) or diagonal ($\bb{A}(x_t) = \text{diag}(a(x_t))$). The authors further suggest that input-dependent, non-diagonal transition matrices can overcome this limitation and can simulate \emph{deterministic finite-state automata} (DFAs). \citet{grazzi2024unlocking} extend this result, demonstrating that SSMs whose transition matrices $\bb{A}(x_t)$ have only positive eigenvalues (typical in Mamba and \gdn) cannot even solve simple state tracking problems like parity. \citet{grazzi2024unlocking} evaluate Mamba-1 \cite{mamba1} and \gdn~\cite{gdn} on $S_3$ and $S_5$ permutation groups. They demonstrate that Mamba-1 and \gdn~can solve the tasks when the eigenvalues of $\bb{A}$ are allowed to be negative (specifically between $[-1, 1]$). \citet{siems2025deltaproduct} further demonstrate that allowing a product of $k - 1$ HouseHolder matrices further increases the expressivity and allows solving tasks like $S_k$. \citet{terzic2025structured} take an alternative approach via structured sparse matrices combined with straight-through gradient estimators, which preserves compatibility with parallel scan \cite{parallelscan,BlellochTR90} while improving expressivity.

\subsubsection{Poor in-context retrieval performance}
Linear RNNs and SSMs significantly underperform parameter-matched Transformer \cite{attention} models on in-context retrieval tasks \cite{gla,deltanet,gdn}. While recent linear RNNs such as DeltaNet \cite{deltanet} and \gdn~\cite{gdn} have closed this performance gap, there still remains a significant gap between Transformers \cite{attention} and linear RNNs. We experimentally demonstrate this further in Section \ref{section:in-context-retrieval-real} (Tables \ref{table:410m-real-retrieval}, \ref{table:7b-real-retrieval}).

\subsection{Non-Linear RNNs}
The most basic non-linear RNN \citep{elman1990finding} takes the form:
\begin{align}
    h_t &= \tanh(\bb{W} h_{t-1} + x_t)\\
    y_t &= h_t
\end{align}
where $h_t \in \mathbb{R}^{d \times 1}$ is the hidden state at time step $t$, $x_t \in \mathbb{R}^{d \times 1}$ is the input, and $\bb{W} \in \mathbb{R}^{d\times d}$ is the (dense, input-independent) transition matrix.

A key challenge for deep non-linear RNNs is vanishing and exploding gradients, which makes learning long-range dependencies difficult \cite{279181}. The Long Short-Term Memory (LSTM) \cite{lstm} and Gated Recurrent Unit (GRU) \cite{cho2014learning} address this by introducing multiplicative gating mechanisms.

\paragraph{Expressivity.} \citet{merrill2019sequential} analyze non-linear RNNs under finite precision and show that they are equivalent to finite-state automata, capable of recognizing all regular languages. Under infinite precision, saturated non-linear RNNs are strictly more expressive, capable of simulating arbitrary deterministic finite-state automata (DFAs) and thus solving hard state-tracking tasks such as permutation group composition that lie outside \tc~\cite{merrill2024illusion}. This expressivity advantage over linear RNNs (which are confined to \tc~under diagonal or input-independent transitions) motivates their use in architectures targeting tasks that require stateful computation beyond what Transformers can express.

\subsection{Limitations of Non-Linear RNNs}
\label{sec:nonlinear-rnn-limitations}
\subsubsection{Poor language modeling performance}
Despite their expressivity advantages, non-linear RNNs like LSTMs \cite{lstm} and GRUs \cite{gru} significantly underperform both Transformers \cite{attention} and modern linear RNNs on language modeling benchmarks. A natural hypothesis is that non-linearity itself is the bottleneck, but we argue the gap is largely attributable to state size. Vector-valued non-linear RNNs maintain a hidden state $h_t \in \mathbb{R}^d$, which is far smaller than the matrix-valued states $\bb{H}_t \in \mathbb{R}^{K \times V}$ used by linear attention models (Section \ref{section:state-size}) for the same number of model parameters. Naively increasing $d$ to match this capacity is impractical: the transition matrix $\bb{W} \in \mathbb{R}^{d \times d}$ grows quadratically in parameters, and for large $d$ it can no longer fit in the shared memory of a single streaming multiprocessor (SM), eliminating the possibility of on-chip reuse.

\subsubsection{Poor in-context retrieval performance}
Beyond language modeling, non-linear RNNs also struggle at in-context retrieval tasks. The fundamental bottleneck is again the state size: a vector-valued hidden state $h_t \in \mathbb{R}^d$ must compress the entire context into a fixed-size representation. As sequence length grows, the state becomes increasingly overloaded, making it difficult to retrieve specific key-value associations stored earlier in the context. This stands in contrast to Transformers, which retain all past tokens in an explicit key-value cache and can attend directly to any position. While linear RNNs with matrix-valued states fare better on retrieval due to their larger effective state capacity, they too lag behind Transformers on in-context retrieval and recall-intensive benchmarks \cite{gla, deltanet, gdn} (Section \ref{section:in-context-retrieval-real}). Non-linear RNNs, constrained to vector-valued states, exhibit an even larger retrieval deficit.

\subsubsection{Training Inefficiency}
\label{sec:training-inefficiency}
\paragraph{Non-Parallelizability on Sequence Length.} The element-wise non-linearity in non-linear RNNs breaks the associativity required by the parallel scan algorithm \cite{parallelscan,BlellochTR90}, forcing sequential computation over the sequence length. DEER \cite{lim2024parallelizingnonlinearsequentialmodels}, ELK \cite{gonzalez2025scalablestableparallelizationnonlinear} and ParaRNN \cite{danieli2025pararnn} show that non-linear RNNs can be parallelized by casting the recurrence as a system of $L$ ($L$ denotes the sequence length) non-linear equations and solving it via Newton's method \cite{newton}. However, storing the Jacobian requires substantial memory and the recurrence becomes computationally expensive, rendering the approach impractical at scale.
To make it feasible, \citet{danieli2025pararnn} impose diagonal structure on the transition matrix, which eliminates cross-neuron mixing and substantially reduces expressivity. Moreover, each forward pass requires multiple Newton iterations, and the number of iterations needed for convergence is not known a priori for a given architecture. Consequently, executing the full sequential recurrence may ultimately be more efficient than iterative Newton updates.

\paragraph{Expensive Computation due to Back-to-back GEMM for Non-Linear RNNs.} Computing the GEMM (General Matrix Multiply) $\bb{W} h_{t-1}$ on a GPU proceeds by tiling the output so that each output tile is assigned to a different streaming multiprocessor (SM). This works well for a single GEMM, as each SM can independently compute its tile without cross-SM communication. However, the tiling becomes problematic at the next time step when computing $\bb{W} h_t$: the output tiles from time step $t-1$ become the input tiles for time step $t$, but they reside on different SMs than where they are needed. Ideally, the SM that produced a given output tile would reuse it directly as input, but the standard tiling prevents this. As a result, intermediate results must be written to and reloaded from global memory (HBM), introducing costly synchronization and data movement. This issue is similar to what one would encounter when fusing two back-to-back GEMMs for MLP computation.

To address this, FlashRNN \cite{flashrnn} modify the RNN recurrence into a block-wise form, analogous to multi-head attention. This design is also adopted by xLSTM \cite{beck2024xlstm}. By keeping the per-block state size small enough to reside in SM shared memory, they avoid HBM-level synchronization and achieve significant speedups.

\paragraph{Poor Hardware Utilization due to Wasted FLOPs.} FlashRNN \cite{flashrnn} propose to parallelize a vector-valued non-linear recurrence over the batch dimension ($B$) and the number of heads ($N$). However, to enable the use of tensor cores on NVIDIA GPUs, \citet{flashrnn} propose to pad the batch dimension $B$ to $B_\textrm{padded}$ ($16$ for using the WMMA instruction) \cite{nvidia_ptx}. This leads to a GEMM computation with shape $\bb{M, N, K} = B_\textrm{padded}, d, d$ at each time step.

In practice, per-GPU batch sizes of $16$ or more are rarely feasible when training large models due to GPU memory limitations (80GB on NVIDIA H100s), leading to substantial wasted FLOPs due to padding. For example, at a batch size of $B = 4$, padding wastes 75\% of the recurrence FLOPs. In Section \ref{sec:hardware-utilization}, we describe how the \method~recurrence avoids this padding.

\subsection{Related Work}
\label{section:related}
\paragraph{xLSTM and FlashRNN.} xLSTM \cite{beck2024xlstm} introduces two recurrent cells: mLSTM has a matrix-valued state but replaces the non-linear recurrence with a linear (associative) update, which enables chunkwise-parallel training while sLSTM keeps a non-linear recurrence but restricts the state to a vector. Neither combines a non-linear recurrence with a large matrix-valued state, which is the combination \method~targets. \method~can be viewed as the simplest parametrization that fills this gap: we take a vector-valued non-linear RNN, expand its state into a matrix via an outer-product write as in linear attention, and add a scalar forget gate that is independent of the recurrent state so that it can be computed in parallel and keeps the per-step computation a single dense GEMM. The same idea can be applied to sLSTM or GRUs. FlashRNN \cite{flashrnn} provides I/O-aware kernels for vector-valued non-linear RNNs (LSTM, GRU, sLSTM). Its kernels map the recurrence onto tensor cores by batching sequences (16 per WMMA call), whereas \method's read, write and gate operations reduce directly to dense $16 \times 16$ or $64 \times 16$ matrix multiplies, independent of the batch size. We compare training throughput against FlashRNN's kernels in \efficiencyref.

\paragraph{Titans.}
Titans \cite{behrouz2024titans} is also a non-linear recurrent architecture, but its non-linearity comes from a generalized MLP memory module rather than the matrix-valued memory used by most linear attention mechanisms and by \method. An MLP-valued memory is possible in principle within \method, but it is significantly more expensive and, in practice, not trainable without the chunkwise update that Titans relies on. The chunkwise update of Titans could in principle be applied to \method~as well, which could significantly improve training throughput at the cost of potentially reduced accuracy.

\paragraph{Parallelizing non-linear RNNs.} DEER \cite{lim2024parallelizingnonlinearsequentialmodels}, ELK \cite{gonzalez2025scalablestableparallelizationnonlinear} and ParaRNN \cite{danieli2025pararnn} parallelize non-linear recurrences over the sequence length by casting them as a system of non-linear equations solved with (quasi-)Newton iterations. These approaches are complementary to \method: they could be combined with the matrix-valued state, although the memory cost of the Jacobian grows with the state size.

\paragraph{Expressivity of Transformers.} The \tc~upper bound applies to fixed-depth Transformers without chain-of-thought. Chain-of-thought \cite{merrill2024expressivepowertransformerschain}, logarithmic depth \cite{merrill2025littledepthgoeslong} or looping \cite{movahedi2026fixedpointreasonersstableadaptive} let Transformers go beyond \tc, at the cost of additional sequential computation.

\paragraph{Expressive linear RNNs.} Linear RNNs with input-dependent, non-diagonal transitions can exceed the expressivity of diagonal ones such as \mamba~and \gdn~\cite{merrill2024illusion, grazzi2024unlocking, siems2025deltaproduct, terzic2025structured}, since iterated products of general unstructured matrices cannot be computed in \tc. Training such models with unstructured transitions requires constraining the spectral norm of the transition at every step to keep the recurrence stable (equivalent to $|\gamma_t| < 1$ for \mamba~and $|1 - \beta_t| < 1$ for \gdn~because of their structured transitions), whereas the $\tanh$ non-linearity bounds the state in non-linear RNNs.

\paragraph{Hybrid architectures.} In our hybrids, attention layers appear to handle exact in-context retrieval while \method~layers contribute non-linear state tracking. Olmo Hybrid \cite{merrill2026olmohybridtheorypracticetheorypractice} proposes a similar specialization between linear recurrent layers and global attention layers. \method~can further improve the state-tracking ability of such hybrids, allowing them to perform state tracking in the \nc~complexity class.

\section{\Method}
In this section, we describe \Method~(\method), our proposed non-linear RNN layer, which addresses the poor language modeling, poor retrieval performance and poor hardware utilization of vector-valued non-linear RNNs while also addressing the limited state tracking problem of linear RNNs.

\subsection{\method~Layer}
To overcome the state size limitations of traditional RNNs, we develop \Method~(\method), which uses matrix-valued hidden states. We adopt an outer product state expansion strategy similar to that of linear attention \cite{katharopoulos2020transformers} and SSMs \cite{mamba1, mamba2}, significantly increasing the state size without significantly increasing the number of parameters. We further incorporate a forget gate which is independent of the recurrent state and can be computed in parallel. We found that adding a residual connection around the \method~recurrence similar to SSMs helps to improve gradient flow and mitigate the compounding effect of the $\tanh$ activation across layers. The forward computation for \method~recurrence is:
\begin{align}
    \label{equation:recurrence}
    \bb{Z}_t &= \tanh(\bb{H}_{t-1} \bb{W} + k_t v_t^\top) \\
    \bb{H}_t &= f_t \bb{H}_{t-1} + (1 - f_t) \bb{Z}_t \\
    y_t &= \bb{H}_t^\top q_t + w_r \odot v_t
\end{align}
where $\bb{H}_{t-1} \in \mathbb{R}^{K \times V}$ denotes the recurrent state at the previous time step. We initialize the initial state $\bb{H}_0 = 0_{K \times V}$ (matrix of shape $K \times V$ filled with 0s). $\bb{Z}_t \in \mathbb{R}^{K \times V}$ denotes the state update and $\bb{H}_t$ denotes the state at the current time step after the update. $\bb{W} \in \mathbb{R}^{V \times V}$ is the transition matrix for the recurrence (independent of the input) while the query, key, value and forget gate ($q_t, k_t \in \mathbb{R}^{K \times 1}, v_t \in \mathbb{R}^{V \times 1}, f_t \in [0, 1]$) are functions of the input $x_t \in \mathbb{R}^{d \times 1}$. $y_t \in \mathbb{R}^{V \times 1}$ represents the final output of the recurrence after the query readout and residual addition and $w_r$ represents the weight for the residual.

\subsubsection{Forget Gate}
\label{sec:forget-gate}
Following LSTMs \cite{lstm} and GRUs \cite{gru}, we augment \method~with a forget gate to mitigate the vanishing gradient problem. We use a scalar-valued (per-head) forget gate $f_t \in [0, 1]$. While we experimented with a vector-valued forget gate of size $K$ per head, it yielded only marginal improvements while substantially increasing the parameter count. The gated recurrent update is
$\bb{H}_t = f_t \bb{H}_{t-1} + (1 - f_t) \bb{Z}_t$,
where $f_t = 1$ fully preserves the current state and $f_t = 0$ resets it entirely. Although the importance of forget gates is well established in the literature \cite{lstm, gru, lstmgruanalysis}, we provide confirmatory ablations in \statesizeref.
Note that the forget gate ($f_t$) is only a function of the input $x_t$ and is independent of the previous state $\bb{H}_{t-1}$ unlike LSTMs \cite{lstm} and GRUs \cite{gru}. The forget gate ($f_{t,n}$)\footnote{We drop the notation for head $f_{t,n}$ in the rest of the paper and use $f_t$ for simplicity.} (forget gate value for $n^{th}$ head at time step $t$) for \method~layers is computed as:
\begin{align}
    f_{t,n} = \psi(x_t) = \frac{1}{(1 + e^{x_t + \beta_n})^{\alpha_n}}
\end{align}
\begin{figure}[h!]
    \centering
    \includegraphics[width=0.5\textwidth]{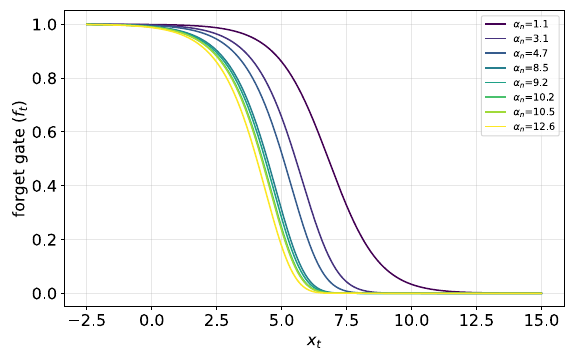}
    \caption{Forget gate ($f_t$) behavior as a function of the input ($x_t$) for different values of $\alpha_n$ and $\beta_n$. $n$ denotes the $n^\textrm{th}$ head.}
    \label{fig:decay}
\end{figure}

where $\alpha_n$ and $\beta_n$ are learnable parameters that are independent of the input.
A smaller value of $\alpha_n$ for the $n^{\text{th}}$ head results in a slower decay rate.
The parameter $\beta_n$ controls the location at which the forget gate saturates (see Figure \ref{fig:decay}). A similar gate is used in Mamba-1 \cite{mamba1}, \mamba~\cite{mamba2} and \gdn~\cite{gdn}. Each head is initialized to a different value of $\alpha_n$ and $\beta_n$, which leads to different forgetting characteristics. We initialize $\alpha_n \sim \textrm{Uniform}(\alpha_\textrm{min}, \alpha_\textrm{max})$ and $\beta_n \sim \textrm{LogUniform}(\beta_\textrm{min}, \beta_\textrm{max})$. Note that when $\alpha_n = 1$, the forget gate simplifies to $\psi(x_t) = \sigma(-x_t - \beta_n)$, which corresponds to a sigmoid with a reflected input and a bias shift.

\subsubsection{\method~Block}
The queries, keys, and values ($q_t, k_t, v_t$) are computed via linear projections followed by a depthwise causal short-convolution (kernel size 4) and SiLU activation \cite{hendrycks2016gaussian, elfwing2018sigmoid}. The output gate $g_t$ uses SiLU without a convolution, and the forget gate $f_t$ uses the parameterized function from \forgetgateref{} to ensure $f_t \in [0, 1]$. Crucially, all pre-recurrence projections depend only on the input $x_t$ and can therefore be computed in parallel across the sequence. The full \method~layer (for a single head) is:
\begin{align}
    \label{equation:layer}
    q_t &= \textrm{SiLU(conv1d}(\bb{W}_q x_t + b_q)) &&\in \mathbb{R}^{K \times 1} \\
    k_t &= \textrm{SiLU(conv1d}(\bb{W}_k x_t + b_k)) &&\in \mathbb{R}^{K \times 1} \\
    v_t &= \textrm{SiLU(conv1d}(\bb{W}_v x_t + b_v)) &&\in \mathbb{R}^{V \times 1} \\
    f_t &= \psi(\bb{W}_f x_t) &&\in [0, 1] \\
    g_t &= \textrm{SiLU}(\bb{W}_g x_t) &&\in \mathbb{R}^{V \times 1} \\
    \bb{Z}_t &= \tanh(\bb{H}_{t-1} \bb{W} + k_t v_t^\top) &&\in \mathbb{R}^{K \times V} \\
    \bb{H}_t &= f_t \bb{H}_{t-1} + (1 - f_t) \bb{Z}_t &&\in \mathbb{R}^{K \times V} \\
    y_t &= \bb{H}_t^\top q_t + w_r \odot v_t &&\in \mathbb{R}^{V \times 1} \\
    y_{gt} &= \textrm{RMSNorm}(y_t \odot g_t) &&\in \mathbb{R}^{V \times 1} \\
    o_t &= \bb{W}_o y_{gt} &&\in \mathbb{R}^{d \times 1} \\
\end{align}
where $x_t \in \mathbb{R}^d$.

Following Mamba-1 \cite{mamba1}, \mamba~\cite{mamba2}, and \gdn~\cite{gdn}, we adopt a hybrid architecture combining the H3 block \cite{fu2022hungry} with a gated MLP, a design that has proven effective in production LLMs (Figure \ref{fig:layer}).

\begin{figure*}[h!]
    \centering
    \begin{tikzpicture}[
    font=\small,
    >={Stealth[length=1.6mm,width=1.4mm]},
    line width=0.4pt,
    block/.style={
        rectangle, rounded corners=1.5pt, draw=black,
        minimum width=14mm, minimum height=5mm, align=center,
        line width=0.5pt
    },
    norm/.style={block, fill=cyan!25},
    method/.style={block, fill=violet!35},
    linear/.style={block, fill=blue!15, minimum width=12mm, minimum height=5mm},
    linstack/.style={block, fill=blue!15, minimum width=12mm, minimum height=5mm},
    conv/.style={block, fill=red!25, minimum width=18mm, minimum height=10mm},
    convstack/.style={block, fill=red!25, minimum width=18mm, minimum height=10mm},
    actsilu/.style={
        circle, draw, fill=violet!18, inner sep=0pt, minimum size=4mm,
        path picture={
            \draw[line width=0.3pt]
                ($(path picture bounding box.west)+(0.3mm,-0.5mm)$)
                -- ($(path picture bounding box.center)+(0,-0.5mm)$)
                .. controls +(0.6mm,0) and +(-0.6mm,-0.4mm) ..
                ($(path picture bounding box.east)+(-0.3mm,0.7mm)$);
        }
    },
    actpsi/.style={circle, draw, fill=violet!18, inner sep=0pt, minimum size=4mm, font=\scriptsize},
    op/.style={circle, draw, inner sep=0pt, minimum size=3.6mm, font=\scriptsize},
    panel/.style={draw=none, inner sep=0pt},
    arr/.style={->, line width=0.45pt},
    lab/.style={font=\scriptsize, inner sep=1pt}
]

\node[norm] (lnorm1) at (0,-5mm) {Norm};
\node[method, above=3mm of lnorm1] (lm2) {M$^2$RNN Layer};
\node[op, above=3mm of lm2] (ladd1) {+};
\node[norm, above=5mm of ladd1] (lnorm2) {Norm};
\node[method, above=3mm of lnorm2] (lmlp) {MLP/MoE};
\node[op, above=3mm of lmlp] (ladd2) {+};

\draw[arr] (lnorm1) -- (lm2);
\draw[arr] (lm2) -- (ladd1);
\draw[arr] (ladd1) -- (lnorm2);
\draw[arr] (lnorm2) -- (lmlp);
\draw[arr] (lmlp) -- (ladd2);

\draw[arr] ($(lnorm1.south)+(0,-7mm)$) -- (lnorm1);
\draw[arr] (ladd2.north) -- ++(0,7mm);

\coordinate (lr1) at ($(lnorm1.south)+(0,-3.5mm)$);
\coordinate (lr2) at ($(ladd1.north)+(0,2.5mm)$);
\draw (lr1) -- ++(-12mm,0) |- (ladd1.west);
\draw (lr2) -- ++(-12mm,0) |- (ladd2.west);

\begin{scope}[on background layer]
    \node[panel, fit=(lnorm1) (ladd2), inner xsep=14mm, inner ysep=8mm] (lpanel) {};
\end{scope}

\coordinate (R) at ($(lpanel.east)+(40mm,-15mm)$);

\node[linear] (rlinG) at ($(R)+(18mm,-22mm)$) {Linear};
\node[linear] (rlinF) at ($(rlinG)+(-14mm,0)$) {Linear};
\node[linear] (rlinQKVf) at ($(rlinF)+(-18mm,0)$) {Linear};
\begin{scope}[on background layer]
    \node[linstack] at ($(rlinQKVf)+(-3.2mm,3.2mm)$) {};
    \node[linstack] at ($(rlinQKVf)+(-1.6mm,1.6mm)$) {};
\end{scope}

\node[conv, above=10mm of rlinQKVf] (rconv) {Causal\\Conv1D};
\begin{scope}[on background layer]
    \node[convstack] at ($(rconv)+(-3.2mm,3.2mm)$) {};
    \node[convstack] at ($(rconv)+(-1.6mm,1.6mm)$) {};
\end{scope}

\node[actsilu] (raq) at ($(rconv.north)+(-7mm,8mm)$) {};
\node[actsilu] (rak) at ($(rconv.north)+(0,8mm)$) {};
\node[actsilu] (rav) at ($(rconv.north)+(7mm,8mm)$) {};
\node[actpsi] (raf) at ($(rlinF.north)+(0,8mm)$) {$\psi$};
\node[actsilu] (rag) at ($(rlinG.north)+(0,8mm)$) {};

\draw[arr] ($(rconv.north)+(-7mm,3.2mm)$) -- (raq);   
\draw[arr] ($(rconv.north)+(0,1.6mm)$)    -- (rak);   
\draw[arr] ($(rconv.north)+(7mm,0)$)      -- (rav);   

\draw[arr] ($(rlinQKVf.north)+(-3.2mm,3.2mm)$) -- ($(rconv.south)+(-3.2mm,0)$);
\draw[arr] ($(rlinQKVf.north)+(-1.6mm,1.6mm)$) -- ($(rconv.south)+(-1.6mm,0)$);
\draw[arr] (rlinQKVf.north) -- (rconv.south);
\draw[arr] (rlinF.north) -- (raf.south);
\draw[arr] (rlinG.north) -- (rag.south);

\node[method, minimum width=30mm, minimum height=6mm] (rrec) at ($(rak)+(6mm,11mm)$) {M$^2$RNN recurrence};

\draw[arr] (raq.north) -- (raq.north |- rrec.south) node[lab,pos=0.55,right=0.4mm] {$q_t$};
\draw[arr] (rak.north) -- (rak.north |- rrec.south) node[lab,pos=0.55,right=0.4mm] {$k_t$};
\draw[arr] (rav.north) -- (rav.north |- rrec.south) node[lab,pos=0.55,right=0.4mm] {$v_t$};
\draw[arr] (raf.north) -- (raf.north |- rrec.south) node[lab,pos=0.55,right=0.4mm] {$f_t$};

\node[op, above=3mm of rrec] (rmul) {$\times$};
\node[norm, above=3mm of rmul] (rrms) {RMSNorm};
\node[linear, above=3mm of rrms] (rfin) {Linear};

\draw[arr] (rrec.north) -- (rmul.south);
\draw[arr] (rmul.north) -- (rrms.south);
\draw[arr] (rrms.north) -- (rfin.south);
\draw[arr] (rfin.north) -- ++(0,6mm);

\draw[arr] (rag.north) |- (rmul.east)
    node[lab, pos=0.05, right=0.4mm] {$g_t$};

\coordinate (rinBus) at ($(rlinF.south)+(0,-4mm)$);
\coordinate (rinBot) at ($(rinBus)+(0,-5mm)$);

\draw[arr] (rlinQKVf.south |- rinBus) -- (rlinQKVf.south);
\draw[arr] (rinBus)                   -- (rlinF.south);
\draw[arr] (rlinG.south |- rinBus)    -- (rlinG.south);

\draw (rlinQKVf.south |- rinBus) -- (rlinG.south |- rinBus);

\draw[arr] (rinBot) -- (rinBus);

\begin{scope}[on background layer]
    \node[rounded corners=2pt, draw=gray!70, line width=0.6pt,
          fit=(rlinQKVf) (rlinG) (rfin) (rmul) (rrec) (rinBus),
          inner xsep=5mm, inner ysep=3mm] (rpanel) {};
\end{scope}

\draw[->, dashed, green!55!black, line width=0.5pt]
    (lm2.east) to[out=20, in=180] ($(rpanel.west)+(0,2mm)$);

\end{tikzpicture}
    \caption{Visualization of the \Method~layer. This block replaces attention and is combined with MLP and RMSNorm modules as in the Transformer.}
    \label{fig:layer}
\end{figure*}
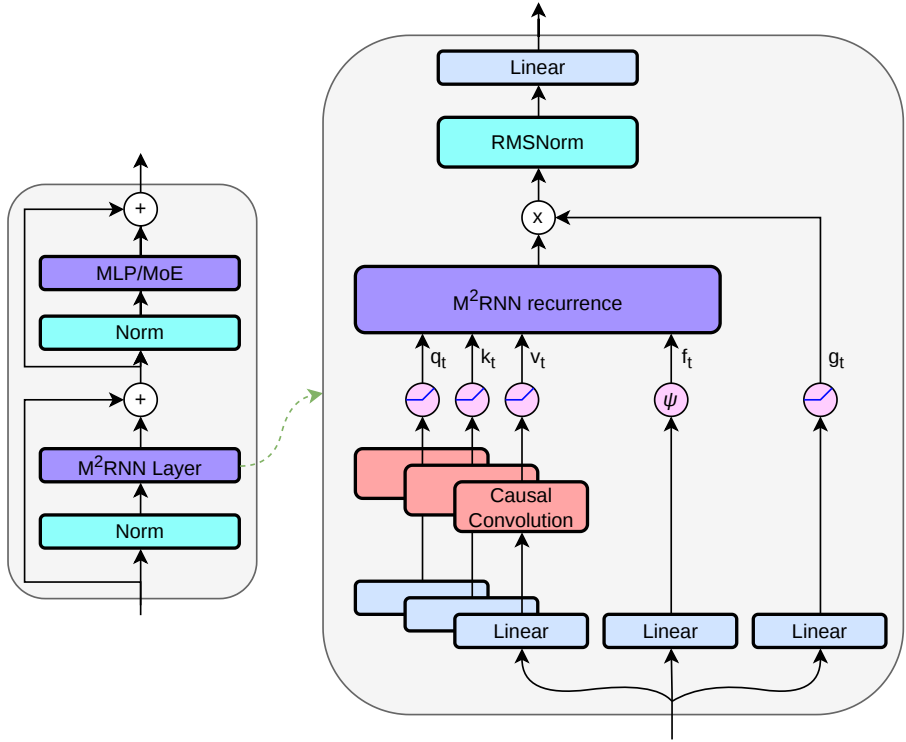

\paragraph{Transition Matrix Initialization.} We experimented with three initialization schemes for the transition matrix $\bb{W}$: normal initialization ($\bb{W} \sim \mathcal{N}(0, \sigma^2)$), orthogonal initialization ($\bb{W}\bb{W}^\top = \bb{I}$), and identity initialization ($\bb{W} = \bb{I}$) \cite{le2015simplewayinitializerecurrent}. Consistent with findings for vector-valued non-linear RNNs, orthogonal initialization outperforms normal initialization. Identity initialization performs on par with orthogonal initialization, so we adopt it for all models.

\paragraph{Gradient Clipping.} Training non-linear RNNs with backpropagation through time (BPTT) is susceptible to exploding gradients, as the gradient of the loss with respect to early hidden states involves repeated multiplication by the transition matrix $\bb{W}$ and the Jacobian of the $\tanh$ nonlinearity. Although the forget gate $f_t$ attenuates gradients flowing through the state, the transition matrix $\bb{W}$ can still cause gradient magnitudes to grow unboundedly. To stabilize training, we apply per-step gradient clipping to the gradient of the recurrent state $\bb{H}_t$ during BPTT.

\subsection{Improved State Tracking}
\begin{tcolorbox}[colback=myblue, colframe=blue, boxrule=0.5pt]
\begin{proposition}
    \label{theorem}
    The \method~recurrence can represent all tasks representable by non-linear vector-valued RNNs and hence can represent regular languages.
\end{proposition}
\end{tcolorbox}
We give the proof for Proposition \ref{theorem} in Appendix \ref{appendix:proof}.

\citet{siems2025deltaproduct} show that Gated DeltaProduct~$[-1,1]$\footnote{Gated DeltaProduct~$[-1,1]$ denotes Gated DeltaProduct with the ability to have negative eigenvalues.} parameterized as a product of $k - 1$ Householder matrices can, in principle, represent the $S_k$ task. In particular, this implies that the $S_3$ task should be representable using Gated DeltaProduct~$[-1,1]$ with product of two Householder matrices. To evaluate this claim empirically, we compare \method, GRU \cite{gru}, \gdn~$[-1,1]$\footnote{\gdn~$[-1,1]$ denotes \gdn~with the ability to have negative eigenvalues.} \cite{gdn}, and Gated DeltaProduct~$[-1,1]$ \cite{siems2025deltaproduct} parameterized as a product of two Householder matrices on the $S_3$ task. All models are trained on sequences of length $128$ and evaluated on context lengths up to $512$.

\begin{figure}[h!]
    \centering
    \includegraphics[width=0.5\textwidth]{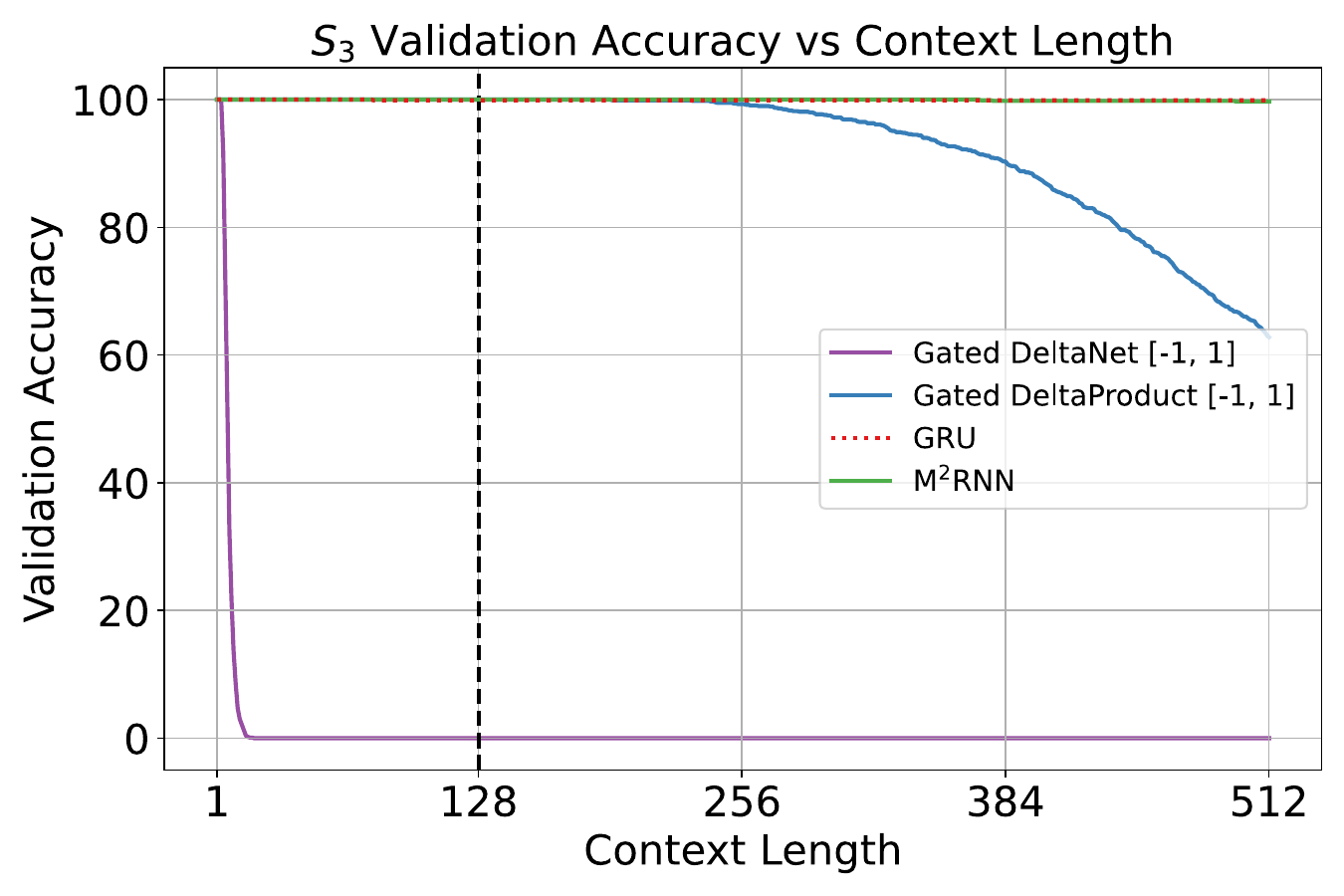}
    \caption{Accuracy on the permutation group $S_3$ state-tracking task for \method, \gdn~$[-1, 1]$, Gated DeltaProduct $[-1, 1]$ and GRU. The vertical line at $128$ denotes the training context length for all models. For the Gated DeltaProduct, we use product of 2 HouseHolder matrices since it can solve the $S_3$ task in theory \cite{grazzi2024unlocking}. However, as the evaluation context length is increased beyond the training context length, we observe accuracy degradation.}
    \label{fig:s3}
\end{figure}

As can be seen from Figure \ref{fig:s3}, \gdn~$[-1, 1]$ fails to learn the task altogether. Gated DeltaProduct~$[-1,1]$ achieves near-perfect accuracy up to context length $256$, however, it fails to generalize to longer sequence lengths not seen during training. These findings suggest that, despite its theoretical expressivity, the model exhibits limited length generalization even on the simple state-tracking $S_3$ task. However, both GRU \cite{gru} and \method~generalize perfectly to unseen context lengths achieving $\ge 99.5\%$ accuracy up to $512$ sequence length. This suggests that theoretical expressivity alone does not guarantee robust state tracking in practice, and that non-linear RNNs like \method~and GRU exhibit stronger length generalization on this task. Motivated by this, we investigate whether non-linear RNNs exhibit this length generalization advantage on language modeling when equipped with an expanded state size.

\subsection{Improved Language Modeling, In-Context Retrieval and Long-Context Performance}
The outer-product formulation for state size expansion in RNNs has become standard in many SSMs and linear RNNs \cite{gla,mamba1,mamba2,deltanet,gdn}. \citet{qin2024hgrn2} demonstrate that larger state sizes improve language modeling performance even when achieved through low-rank parameterizations. We also demonstrate this with empirical evidence in Section \ref{section:state-size}.

Additionally, \mamba~\cite{mamba2} increases the effective state size through a grouping strategy that shares parameters across heads using a multi-value attention formulation, allowing the model to scale the state dimension without increasing the overall parameter count. Building on this idea, we use an outer product with the multi-value attention formulation proposed by \citet{mamba2} to increase the achievable state size under a fixed parameter budget (see Table \ref{table:param-state} in Section \ref{section:state-size}).

The expanded matrix-valued state directly addresses the two core weaknesses of vector-valued non-linear RNNs identified in Section~\ref{sec:nonlinear-rnn-limitations}: poor language modeling and poor retrieval performance.

\paragraph{Language modeling.} The low language modeling performance of non-linear RNNs such as LSTMs \cite{lstm} and GRUs \cite{gru} relative to Transformers \cite{attention} and linear RNNs has traditionally been attributed to the non-linearity itself. We argue instead that the primary bottleneck is state capacity. By expanding the hidden state from a vector $h_t \in \mathbb{R}^d$ to a matrix $\bb{H}_t \in \mathbb{R}^{K \times V}$, \method~dramatically increases the information that can be stored at each recurrent step without a proportional increase in parameter count. We provide empirical evidence for this in Section~\ref{section:state-size}, showing that larger state sizes yield consistent language modeling improvements under a fixed parameter budget; it further varies the state size at a fixed parameter count and ablates the forget gate.

\paragraph{In-context retrieval and long-context performance.} The outer product write $k_t v_t^\top$ significantly expands the state size enabling efficient storage of key-value associations and structured in-context retrieval via the readout $\bb{H}_t^\top q_t$. This capacity is absent in vector-valued non-linear RNNs, which cannot store enough associations due to the limited state size. We demonstrate the resulting gains in Tables \ref{table:410m-real-retrieval}, \ref{table:7b-real-retrieval}. We further demonstrate the impact of the state size expansion in long-context tasks in Tables \ref{table:410m-longbench}, \ref{table:7b-longbench}.

\subsection{Improved Hardware Utilization}
\label{sec:hardware-utilization}
\subsubsection{Improved Tiling}
The outer product expansion enables \method\ to be efficiently parallelized across both the batch dimension $B$ and the number of heads $N$. In this formulation, each GPU streaming multiprocessor (SM) executes an independent recurrence over a matrix-valued state of size $K \times V$. Each step of the recurrence involves a GEMM with dimensions $\bb{M, N, K} = K, V, V$.

Importantly, the matrix-valued recurrence in \method\ requires no padding as long as $K$ and $V$ are multiples of 16 allowing utilizing the tensor cores. For example, when $K = 64$ and $V = 16$, the GEMM computation $\bb{H}_{t-1} \bb{W}$ can utilize the WGMMA instruction on NVIDIA Hopper GPUs \cite{nvidia_ptx}, ensuring high hardware efficiency.

Crucially, the GEMM shape in \method~is independent of the batch size $B$. This contrasts with vector-valued non-linear RNNs, where GEMM dimensions depend on $B$, often necessitating batch-size padding to achieve tensor core compatibility. As a result, the outer product state expansion mechanism achieves improved parallelism without padding, at the cost of more FLOPs per step than vector-valued RNNs,, avoiding the padding overhead described in Section \ref{sec:training-inefficiency}.

\subsubsection{Efficient Integration of \method~Layers in Hybrid Architectures}
Due to the expensive nature of \method~layers, we explore using them more sparingly. Our results show that replacing only a single \mamba~or \gdn~layer with \method~in a hybrid architecture achieves the same accuracy improvements as the full hybrid \method~model, while keeping training throughput within 6\% of the Hybrid \gdn~at both 4k and 16k context lengths. This provides a scalable strategy for incorporating \method~layers into large language models.

\section{Distributed Training and Systems Optimizations}
In this section, we discuss the system considerations for \method~layers, including the forward and backward algorithms (Section \ref{sec:forward}, \ref{sec:backward}) and tensor parallelism (Section \ref{sec:tensor-parallel}).

\begin{algorithm}[h!]
    \caption{Forward Computation for \method~layer with multi-value attention formulation. The \blue{blue colour} indicates computation or tensors residing in on-chip memory.}
    \label{algorithm:forward}

    \begin{algorithmic}
        \INPUT: $\bb{H}_0 \in \mathbb{R}^{BNKV}, \bb{Q}, \bb{K} \in \mathbb{R}^{BTK}, \bb{V} \in \mathbb{R}^{BTNV}, \bb{W} \in \mathbb{R}^{NVV}, \bb{F} \in [0, 1]^{BTN}$
        \OUTPUT: $\bb{Y} \in \mathbb{R}^{BTNV}$

        \FOR[Each head on a separate SM]{$n = 0$ to $N - 1$ \textbf{in parallel}}
            \STATE Load $\bb{\blue{W}} \leftarrow \bb{W}[n, :, :]$ from HBM \COMMENT{Transition matrix to SRAM (reused across batch)}

            \FOR{$b = 0$ to $B - 1$ \textbf{in parallel}}
                \STATE Load $\bb{\blue{H}} \leftarrow \bb{H}_0[b, n, :, :]$ from HBM \COMMENT{Initial hidden state ($K \times V$) to SRAM}

                \FOR[Sequential recurrence over time]{t = $0$ to $T - 1$}
                    \STATE Load $\blue{q, k, v, f} \leftarrow \bb{Q}[b,t,:], \bb{K}[b,t,:], \bb{V}[b,t,n, :], \bb{F}[b,t,n]$ from HBM

                    \STATE Compute $\blue{\bb{Z} \leftarrow \tanh(\bb{H} \bb{W} + k v^\top)}$ \COMMENT{Nonlinear transition with outer product input}
                    \STATE Compute $\blue{\bb{H} \leftarrow f \bb{H} + (1 - f) \bb{Z}}$ \COMMENT{Gated update of hidden state}
                    \STATE Compute $\blue{y \leftarrow \bb{H}^\top q}$ \COMMENT{Query readout from hidden state}

                    \STATE Store $\bb{Y}[b,t,n,:] \leftarrow \blue{y}$ to HBM \COMMENT{Write output; no intermediate $\bb{H}$ cached}
                \ENDFOR
            \ENDFOR
        \ENDFOR
    \end{algorithmic}
\end{algorithm}

\begin{algorithm}[h!]
    \caption{Backward Computation for \method~layer with multi-value attention formulation. The \blue{blue colour} indicates computation or tensors residing in on-chip memory.}
    \label{algorithm:backward}

    \begin{algorithmic}
        \INPUT: $\bb{H}_0 \in \mathbb{R}^{BNKV}, \bb{Q} \in \mathbb{R}^{BTK}, \bb{K} \in \mathbb{R}^{BTK}, \bb{V} \in \mathbb{R}^{BTNV}, \bb{W} \in \mathbb{R}^{NVV}, \bb{F} \in [0, 1]^{BTN}, \bb{H}_\text{full} \in \mathbb{R}^{BTNKV}$
        \OUTPUT: $d\bb{Q} \in \mathbb{R}^{BTK}, d\bb{K} \in \mathbb{R}^{BTK}, d\bb{V} \in \mathbb{R}^{BTNV}, d\bb{W} \in \mathbb{R}^{NVV}, d\bb{F} \in [0, 1]^{BTN}$

        \FOR[Each head on a separate SM]{$n = 0$ to $N - 1$ \textbf{in parallel}}
            \STATE Load $\bb{\blue{W}} \leftarrow \bb{W}[n, :, :]$ from HBM \COMMENT{Transition matrix to SRAM}

            \FOR{$b = 0$ to $B - 1$ \textbf{in parallel}}
                \STATE Load $\bb{\blue{H}} \leftarrow \bb{H}_0[b, n, :, :]$ from HBM \COMMENT{Initial hidden state to SRAM}

                \FOR[Pass 1: forward recomputation, cache all $\bb{H}_t$]{$t = 0$ to $T - 1$}
                    \STATE Load $\blue{k, v, f} \leftarrow \bb{K}[b,t, :], \bb{V}[b,t,n, :], \bb{F}[b,t,n]$ from HBM

                    \STATE Store $\bb{H}_\text{full}[b, t, n, :, :] \leftarrow \blue{f \bb{H} + (1 - f) \tanh(\bb{H} \bb{W} + k v^\top)}$ to HBM \COMMENT{Cache hidden state for backward}
                \ENDFOR

                \STATE \blue{$d\bb{W} \leftarrow 0_{V \times V}$} \COMMENT{Initialize $d\bb{W}$ to 0}

                \FOR[Pass 2: reverse sweep to compute gradients]{$t = T - 1$ to $0$}
                    \IF{$t == 0$}
                        \STATE Load $\blue{\bb{H}_\text{prev}} \leftarrow \bb{H}_0[b, n, :, :]$ from HBM
                    \ELSE
                        \STATE Load $\blue{\bb{H}_\text{prev}} \leftarrow \bb{H}_\text{full}[b, t - 1, n, :, :]$ from HBM
                    \ENDIF

                    \STATE Load $\blue{q, k, v, f, dy} \leftarrow \bb{Q}[b,t, :], \bb{K}[b,t, :], \bb{V}[b,t,n, :], \bb{F}[b,t,n], \bb{dY}[b,t,n, :]$ from HBM
                    
                    \STATE Compute $\blue{\bb{Z} \leftarrow \tanh(\bb{H}_\text{prev} \bb{W} + k_t v_t^\top)}$ \COMMENT{Recompute pre-gate activation}
                    \STATE Compute $\blue{\bb{H} \leftarrow f \bb{H}_\text{prev} + (1 - f) \bb{Z}}$ \COMMENT{Recompute hidden state at $t$}

                    \STATE Compute $\blue{dq \leftarrow \bb{H} \cdot dy}$ \COMMENT{Gradient w.r.t.\ query}
                    \STATE Atomic Add $d\bb{Q}[b, t, :] \leftarrow \blue{dq}$ to HBM \COMMENT{Accumulate across shared heads}

                    \STATE Compute $\blue{r \leftarrow q \cdot dy^\top}$ \COMMENT{Outer product for temporary variable}
                    \STATE Compute $\blue{df \leftarrow 1_K^\top \left[r \odot (\bb{H}_\text{prev} - \bb{Z}) \right]1_V}$ \COMMENT{Gradient w.r.t.\ forget gate (scalar)}
                    \STATE Store $d\bb{F}[b, t, n] \leftarrow \blue{df}$ to HBM

                    \STATE Compute $\blue{d\bb{H} \leftarrow r \odot f}$ \COMMENT{Gradient through the gated residual}
                    \STATE Compute $\blue{d\bb{Z} \leftarrow r \odot (1 - f)}$ \COMMENT{Gradient through the nonlinear branch}
                    
                    \STATE Compute $\blue{d\bb{X} \leftarrow d\bb{Z} \odot (\bb{1}_{K \times V} - \bb{Z} \odot \bb{Z})}$ \COMMENT{Backprop through $\tanh$}
                    \STATE Compute $\blue{d\bb{H} \leftarrow d\bb{Z} \odot (d\bb{X} \cdot \bb{W}^\top)}$ \COMMENT{Gradient w.r.t.\ previous hidden state}
                    \STATE Compute $\blue{d\bb{W} \leftarrow d\bb{W} + \bb{H}_\text{prev}^\top d\bb{X}}$ \COMMENT{Accumulate gradient for transition matrix}

                    \STATE Compute $\blue{dk \leftarrow d\bb{X} \cdot v}$ \COMMENT{Gradient w.r.t.\ key}
                    \STATE Atomic Add $d\bb{K}[b, t, :] \leftarrow \blue{dk}$ to HBM \COMMENT{Accumulate across shared heads}

                    \STATE Compute $\blue{dv \leftarrow d\bb{X}^\top k}$ \COMMENT{Gradient w.r.t.\ value}
                    \STATE Store $d\bb{V}[b, t, n, :] \leftarrow \blue{dv}$ to HBM \COMMENT{Write to HBM}
                \ENDFOR
            \STATE Atomic Add $d\bb{W}[n, :, :] \leftarrow \blue{d\bb{W}}$ to HBM \COMMENT{Accumulate $d\bb{W}$ across batch}
            \ENDFOR
        \ENDFOR
    \end{algorithmic}
\end{algorithm}

\subsection{Forward Algorithm}
\label{sec:forward}
We implement the forward computation for the \method~layer using the Triton \cite{triton} DSL. The kernel parallelizes across the batch dimension $B$ and the number of heads $N$, assigning each $(b, n)$ pair to a separate GPU SM. 
We illustrate our algorithm in Algorithm \ref{algorithm:forward} for the case of multi-value formulation ($N_q = N_k = 1$ and $N_v, N_f, N_w = N$) which we use for all our models. We avoid caching intermediate activations during the forward pass for use in backward computation, as storing the $\bb{H}$ matrix is prohibitively expensive in both I/O and memory footprint: $\bb{H}_t \in \mathbb{R}^{BNKV}$ is much larger than the per-time-step keys and values ($k_t \in \mathbb{R}^{N_kK}, v_t \in \mathbb{R}^{N_vV}$).

\subsection{Backward Algorithm}
\label{sec:backward}
For the backward computation, we recompute the forward recurrence without the final query readout and cache $\bb{H}_t$ at every time step to HBM. The backward kernel then reads $\bb{H}_t$ at every time step from HBM to compute the gradients for the inputs. The backward kernel is implemented in the Triton DSL \cite{triton}, which does not expose direct programmability of shared memory. As a result, the kernel is heavily memory-bandwidth bound; an optimized CUTLASS implementation is left for future work. Note that this means the backward pass does materialize all $\bb{H}_t$ in HBM, costing $O(BNKVT)$ memory and I/O during the backward pass for a sequence of length $T$; avoiding caching only applies to the forward pass. Ideally, the recomputation would be kept entirely on-chip in SRAM, but no chunkwise-parallel form exists for a non-linear recurrence, so the backward pass must reconstruct $\bb{H}_t$ by recomputation in reverse order, and Triton does not expose the shared-memory control needed to keep the recomputed states on-chip. Kernels written in TileLang or CuTe-DSL could keep (chunks of) $\bb{H}_t$ in SRAM, although caching multiple states per chunk would likely saturate the available SRAM. This limitation is caused by the combination of the non-linearity (which rules out associative recomputation) and the Triton implementation, not by the non-diagonal transition. Diagonal linear RNNs such as Mamba-1 avoid this cost by recomputing their states on-chip with a parallel scan \cite{mamba1}. On Hopper GPUs, intermediates could also be exchanged between SMs through distributed shared memory within thread-block clusters, which is a possible direction for future \method~kernels. The algorithm is illustrated in Algorithm \ref{algorithm:backward}.

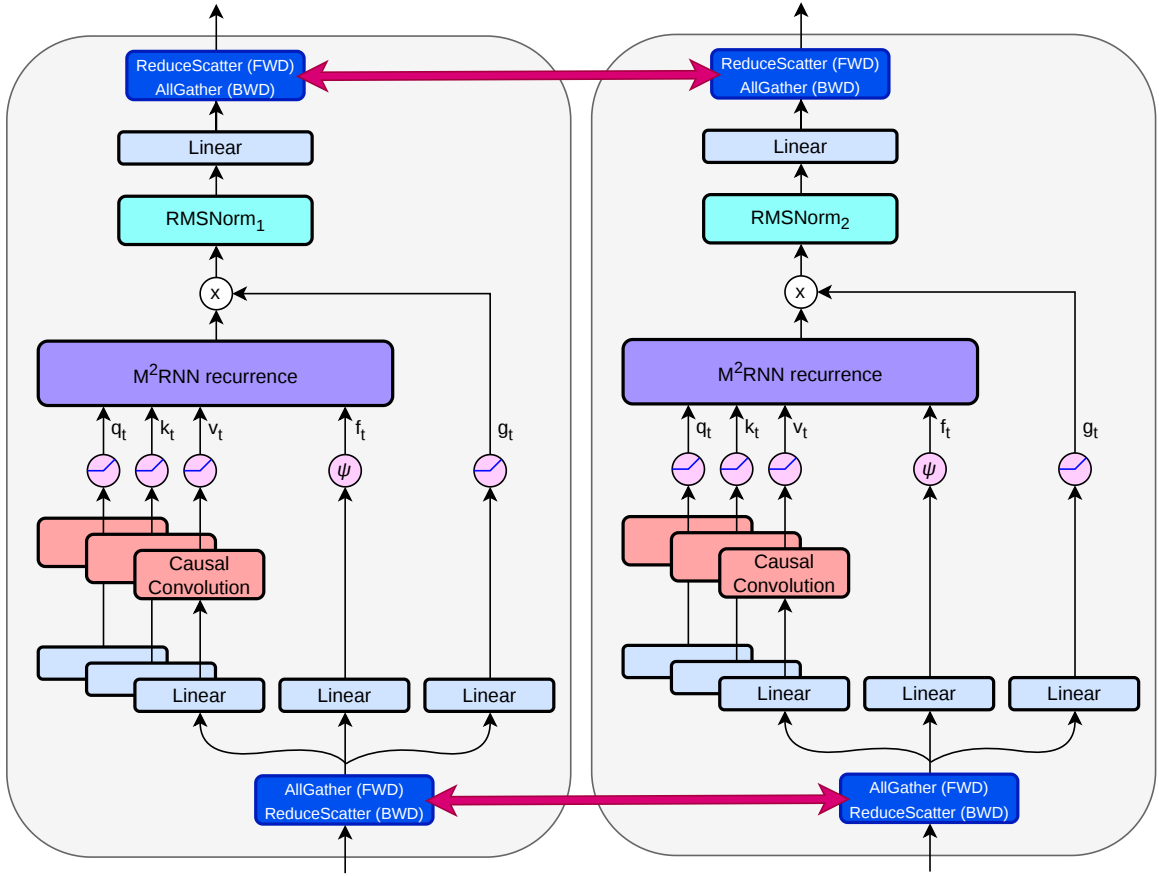
\begin{figure*}[h!]
    \centering
    \begin{tikzpicture}[
    font=\small,
    >={Stealth[length=1.6mm,width=1.4mm]},
    line width=0.4pt,
    block/.style={
        rectangle, rounded corners=1.5pt, draw=black,
        minimum width=14mm, minimum height=5mm, align=center,
        line width=0.5pt
    },
    norm/.style={block, fill=cyan!25},
    method/.style={block, fill=violet!35},
    linear/.style={block, fill=blue!15, minimum width=12mm, minimum height=5mm},
    comms/.style={block, fill=green!15, minimum width=12mm, minimum height=5mm},
    linstack/.style={block, fill=blue!15, minimum width=12mm, minimum height=5mm},
    conv/.style={block, fill=red!25, minimum width=18mm, minimum height=10mm},
    convstack/.style={block, fill=red!25, minimum width=18mm, minimum height=10mm},
    actsilu/.style={
        circle, draw, fill=violet!18, inner sep=0pt, minimum size=4mm,
        path picture={
            \draw[line width=0.3pt]
                ($(path picture bounding box.west)+(0.3mm,-0.5mm)$)
                -- ($(path picture bounding box.center)+(0,-0.5mm)$)
                .. controls +(0.6mm,0) and +(-0.6mm,-0.4mm) ..
                ($(path picture bounding box.east)+(-0.3mm,0.7mm)$);
        }
    },
    actpsi/.style={circle, draw, fill=violet!18, inner sep=0pt, minimum size=4mm, font=\scriptsize},
    op/.style={circle, draw, inner sep=0pt, minimum size=3.6mm, font=\scriptsize},
    panel/.style={draw=none, inner sep=0pt},
    arr/.style={->, line width=0.45pt},
    lab/.style={font=\scriptsize, inner sep=1pt}
]

\coordinate (L) at (0,0);

\node[linear] (llinG) at ($(L)+(18mm,-22mm)$) {Linear};
\node[linear] (llinF) at ($(llinG)+(-14mm,0)$) {Linear};
\node[linear] (llinQKVf) at ($(llinF)+(-18mm,0)$) {Linear};
\begin{scope}[on background layer]
    \node[linstack] at ($(llinQKVf)+(-3.2mm,3.2mm)$) {};
    \node[linstack] at ($(llinQKVf)+(-1.6mm,1.6mm)$) {};
\end{scope}

\node[conv, above=10mm of llinQKVf] (lconv) {Causal\\Conv1D};
\begin{scope}[on background layer]
    \node[convstack] at ($(lconv)+(-3.2mm,3.2mm)$) {};
    \node[convstack] at ($(lconv)+(-1.6mm,1.6mm)$) {};
\end{scope}

\node[actsilu] (laq) at ($(lconv.north)+(-7mm,8mm)$) {};
\node[actsilu] (lak) at ($(lconv.north)+(0,8mm)$) {};
\node[actsilu] (lav) at ($(lconv.north)+(7mm,8mm)$) {};
\node[actpsi] (laf) at ($(llinF.north)+(0,8mm)$) {$\psi$};
\node[actsilu] (lag) at ($(llinG.north)+(0,8mm)$) {};

\draw[arr] ($(lconv.north)+(-7mm,3.2mm)$) -- (laq);   
\draw[arr] ($(lconv.north)+(0,1.6mm)$)    -- (lak);   
\draw[arr] ($(lconv.north)+(7mm,0)$)      -- (lav);   

\draw[arr] ($(llinQKVf.north)+(-3.2mm,3.2mm)$) -- ($(lconv.south)+(-3.2mm,0)$);
\draw[arr] ($(llinQKVf.north)+(-1.6mm,1.6mm)$) -- ($(lconv.south)+(-1.6mm,0)$);
\draw[arr] (llinQKVf.north) -- (lconv.south);
\draw[arr] (llinF.north) -- (laf.south);
\draw[arr] (llinG.north) -- (lag.south);

\node[method, minimum width=30mm, minimum height=6mm] (lrec) at ($(lak)+(6mm,11mm)$) {M$^2$RNN recurrence};

\draw[arr] (laq.north) -- (laq.north |- lrec.south) node[lab,pos=0.55,right=0.4mm] {$q_t$};
\draw[arr] (lak.north) -- (lak.north |- lrec.south) node[lab,pos=0.55,right=0.4mm] {$k_t$};
\draw[arr] (lav.north) -- (lav.north |- lrec.south) node[lab,pos=0.55,right=0.4mm] {$v_t$};
\draw[arr] (laf.north) -- (laf.north |- lrec.south) node[lab,pos=0.55,right=0.4mm] {$f_t$};

\node[op, above=3mm of lrec] (lmul) {$\times$};
\node[norm, above=3mm of lmul] (lrms) {RMSNorm};
\node[linear, above=3mm of lrms] (lfin) {Linear};

\draw[arr] (lrec.north) -- (lmul.south);
\draw[arr] (lmul.north) -- (lrms.south);
\draw[arr] (lrms.north) -- (lfin.south);

\draw[arr] (lag.north) |- (lmul.east)
    node[lab, pos=0.05, right=0.4mm] {$g_t$};

\coordinate (linBus) at ($(llinF.south)+(0,-4mm)$);
\coordinate (linBot) at ($(linBus)+(0,-5mm)$);

\draw[arr] (llinQKVf.south |- linBus) -- (llinQKVf.south);
\draw[arr] (linBus)                   -- (llinF.south);
\draw[arr] (llinG.south |- linBus)    -- (llinG.south);

\draw (llinQKVf.south |- linBus) -- (llinG.south |- linBus);

\begin{scope}[on background layer]
    \node[rounded corners=2pt, draw=gray!70, line width=0.6pt,
          fit=(llinQKVf) (llinG) (lfin) (lmul) (lrec) (linBus),
          inner xsep=5mm, inner ysep=3mm] (lpanel) {};
\end{scope}

\coordinate (R) at ($(lpanel.east|-L)+(35mm,0)$);

\node[linear] (rlinG) at ($(R)+(18mm,-22mm)$) {Linear};
\node[linear] (rlinF) at ($(rlinG)+(-14mm,0)$) {Linear};
\node[linear] (rlinQKVf) at ($(rlinF)+(-18mm,0)$) {Linear};
\begin{scope}[on background layer]
    \node[linstack] at ($(rlinQKVf)+(-3.2mm,3.2mm)$) {};
    \node[linstack] at ($(rlinQKVf)+(-1.6mm,1.6mm)$) {};
\end{scope}

\node[conv, above=10mm of rlinQKVf] (rconv) {Causal\\Conv1D};
\begin{scope}[on background layer]
    \node[convstack] at ($(rconv)+(-3.2mm,3.2mm)$) {};
    \node[convstack] at ($(rconv)+(-1.6mm,1.6mm)$) {};
\end{scope}

\node[actsilu] (raq) at ($(rconv.north)+(-7mm,8mm)$) {};
\node[actsilu] (rak) at ($(rconv.north)+(0,8mm)$) {};
\node[actsilu] (rav) at ($(rconv.north)+(7mm,8mm)$) {};
\node[actpsi] (raf) at ($(rlinF.north)+(0,8mm)$) {$\psi$};
\node[actsilu] (rag) at ($(rlinG.north)+(0,8mm)$) {};

\draw[arr] ($(rconv.north)+(-7mm,3.2mm)$) -- (raq);   
\draw[arr] ($(rconv.north)+(0,1.6mm)$)    -- (rak);   
\draw[arr] ($(rconv.north)+(7mm,0)$)      -- (rav);   

\draw[arr] ($(rlinQKVf.north)+(-3.2mm,3.2mm)$) -- ($(rconv.south)+(-3.2mm,0)$);
\draw[arr] ($(rlinQKVf.north)+(-1.6mm,1.6mm)$) -- ($(rconv.south)+(-1.6mm,0)$);
\draw[arr] (rlinQKVf.north) -- (rconv.south);
\draw[arr] (rlinF.north) -- (raf.south);
\draw[arr] (rlinG.north) -- (rag.south);

\node[method, minimum width=30mm, minimum height=6mm] (rrec) at ($(rak)+(6mm,11mm)$) {M$^2$RNN recurrence};

\draw[arr] (raq.north) -- (raq.north |- rrec.south) node[lab,pos=0.55,right=0.4mm] {$q_t$};
\draw[arr] (rak.north) -- (rak.north |- rrec.south) node[lab,pos=0.55,right=0.4mm] {$k_t$};
\draw[arr] (rav.north) -- (rav.north |- rrec.south) node[lab,pos=0.55,right=0.4mm] {$v_t$};
\draw[arr] (raf.north) -- (raf.north |- rrec.south) node[lab,pos=0.55,right=0.4mm] {$f_t$};

\node[op, above=3mm of rrec] (rmul) {$\times$};
\node[norm, above=3mm of rmul] (rrms) {RMSNorm};
\node[linear, above=3mm of rrms] (rfin) {Linear};

\draw[arr] (rrec.north) -- (rmul.south);
\draw[arr] (rmul.north) -- (rrms.south);
\draw[arr] (rrms.north) -- (rfin.south);

\draw[arr] (rag.north) |- (rmul.east)
    node[lab, pos=0.05, right=0.4mm] {$g_t$};

\coordinate (rinBus) at ($(rlinF.south)+(0,-4mm)$);
\coordinate (rinBot) at ($(rinBus)+(0,-5mm)$);

\draw[arr] (rlinQKVf.south |- rinBus) -- (rlinQKVf.south);
\draw[arr] (rinBus)                   -- (rlinF.south);
\draw[arr] (rlinG.south |- rinBus)    -- (rlinG.south);

\draw (rlinQKVf.south |- rinBus) -- (rlinG.south |- rinBus);

\begin{scope}[on background layer]
    \node[rounded corners=2pt, draw=gray!70, line width=0.6pt,
          fit=(rlinQKVf) (rlinG) (rfin) (rmul) (rrec) (rinBus),
          inner xsep=5mm, inner ysep=3mm] (rpanel) {};
\end{scope}

\node[comms] (top) at ($(lfin)!0.5!(rfin)+(0,12mm)$) {ReduceScatter (FWD)\\AllGather (BWD)};
\draw[arr] (lfin.north) |- (top.west);
\draw[arr] (rfin.north) |- (top.east);
\draw[arr] (top.north) -- ++(0,6mm);

\node[comms] (bot) at ($(top|-linBus)+(0,-12mm)$) {AllGather (FWD)\\ReduceScatter (BWD)};
\draw[arr] ($(bot.south)+(0,-6mm)$) -- (bot.south);
\draw[arr] (bot.west) -| (linBus);
\draw[arr] (bot.east) -| (rinBus);

\end{tikzpicture}
    \caption{TP topology-aware \method~layer running with TP on 2 GPUs. Note that RMSNorm$_1$ and RMSNorm$_2$ have different weights for the RMSNorm module on both GPUs.}
    \label{fig:tp-aware}
\end{figure*}

\subsection{Tensor Parallelism}
\label{sec:tensor-parallel}
\citet{shoeybi2019megatron} introduced tensor parallelism (TP), a model parallelism technique that partitions attention and MLP layers across multiple accelerators. This approach has become standard for training large-scale models \cite{palm,megatron-turing,llama3,gpt3,workshop2022bloom}. The core challenge in applying TP for \method~is that the multi-value formulation shares a single query and key head across all value heads ($N_q = N_k = 1$, $N_v = N$), preventing a straightforward partition of heads across GPUs. We present two strategies for implementing TP in \method~layers: a grouped-value approach that avoids extra communication, and a multi-value approach that preserves the parameter count at the cost of additional communication during training. We represent the TP world size as $N_\textrm{TP}$.

\subsubsection{TP Topology-Aware \method~Layers}
\label{sec:tp-topology-aware}
A straightforward way to implement TP for \method~layers is to adopt a grouped-value formulation: the number of query and key heads is set equal to the TP world size ($N_q = N_k = N_\textrm{TP}$), while the number of value heads remains $N_v = N$. Each GPU is assigned one query head, one key head, and $\frac{N_v}{N_\textrm{TP}}$ value heads\footnote{We assume $N_v$ is divisible by $N_\textrm{TP}$.}, and applies RMSNorm~\cite{rmsnorm} with different weights on different GPUs in the TP group (refer to Figure \ref{fig:tp-aware}). Since this is equivalent to running a multi-value \method~recurrence on every GPU, this requires no modifications to the kernels.

\paragraph{Advantages.} Each GPU operates independently on its assigned heads, so no communication is introduced beyond the standard TP communication required for an attention layer.

\paragraph{Disadvantages.} The query and key projection sizes grow proportionally with $N_\textrm{TP}$, coupling the parameter count to the TP topology used during training\footnote{The RMSNorm weight size also grows with $N_\textrm{TP}$ but is significantly smaller in comparison.}. Consequently, a model trained with TP world size $N_\textrm{TP}$ cannot be straightforwardly deployed at a smaller world size without carrying the inflated query, key projections and the larger activation footprint that comes with them (specifically, the input to the output linear projection). This is analogous to the approach proposed in \mamba~\cite{mamba2} and the Desync-Residual strategy\footnote{Refer to \url{https://arxiv.org/abs/2501.06589v2} (v2) since later versions of the paper dropped Desync-Residual.} \cite{zhang2025ladderresidualparallelismawarearchitectureaccelerating}.

\subsubsection{TP Topology-Independent \method~Layers}
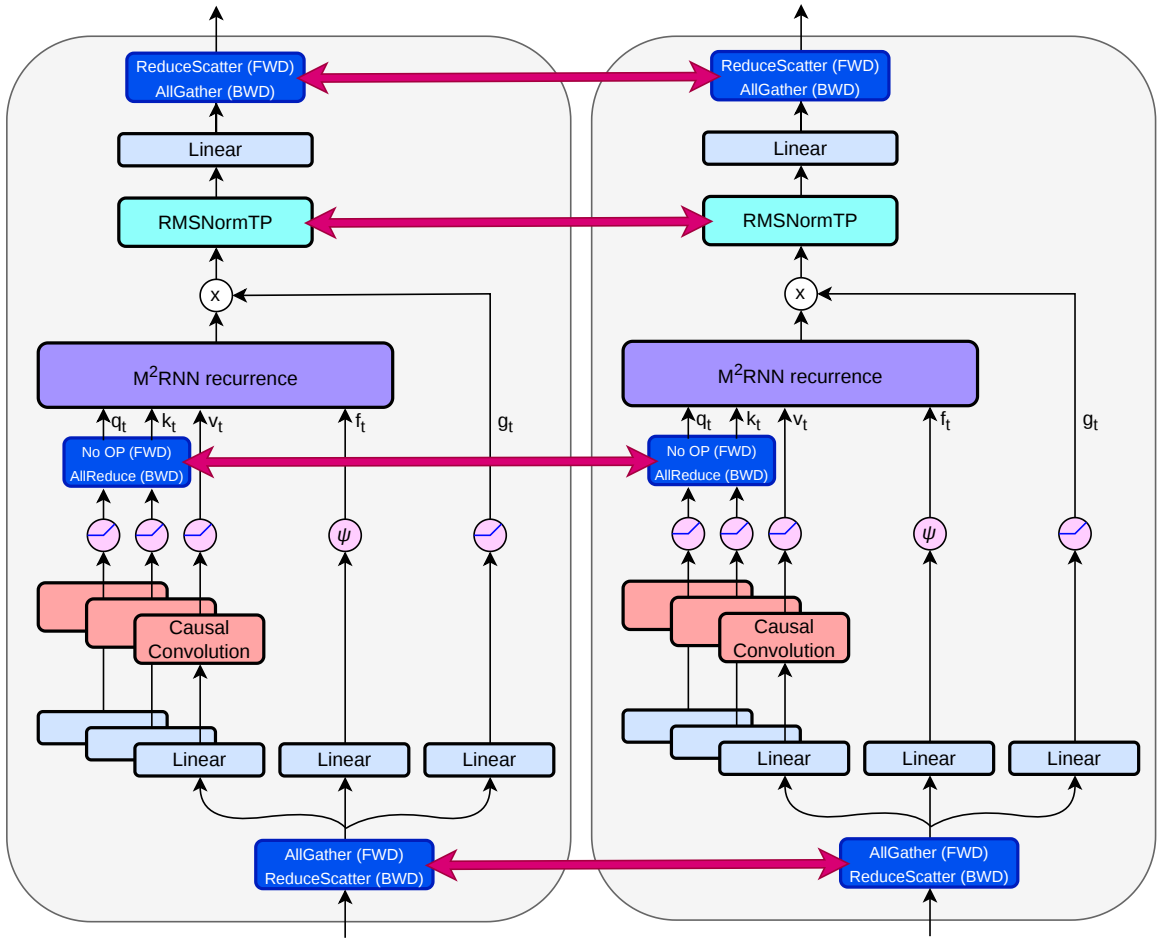
\begin{figure*}[h!]
    \centering
    \begin{tikzpicture}[
    font=\small,
    >={Stealth[length=1.6mm,width=1.4mm]},
    line width=0.4pt,
    block/.style={
        rectangle, rounded corners=1.5pt, draw=black,
        minimum width=14mm, minimum height=5mm, align=center,
        line width=0.5pt
    },
    norm/.style={block, fill=cyan!25},
    method/.style={block, fill=violet!35},
    linear/.style={block, fill=blue!15, minimum width=12mm, minimum height=5mm},
    comms/.style={block, fill=green!15, minimum width=12mm, minimum height=5mm},
    linstack/.style={block, fill=blue!15, minimum width=12mm, minimum height=5mm},
    conv/.style={block, fill=red!25, minimum width=18mm, minimum height=10mm},
    convstack/.style={block, fill=red!25, minimum width=18mm, minimum height=10mm},
    actsilu/.style={
        circle, draw, fill=violet!18, inner sep=0pt, minimum size=4mm,
        path picture={
            \draw[line width=0.3pt]
                ($(path picture bounding box.west)+(0.3mm,-0.5mm)$)
                -- ($(path picture bounding box.center)+(0,-0.5mm)$)
                .. controls +(0.6mm,0) and +(-0.6mm,-0.4mm) ..
                ($(path picture bounding box.east)+(-0.3mm,0.7mm)$);
        }
    },
    actpsi/.style={circle, draw, fill=violet!18, inner sep=0pt, minimum size=4mm, font=\scriptsize},
    op/.style={circle, draw, inner sep=0pt, minimum size=3.6mm, font=\scriptsize},
    panel/.style={draw=none, inner sep=0pt},
    arr/.style={->, line width=0.45pt},
    lab/.style={font=\scriptsize, inner sep=1pt}
]

\coordinate (L) at (0,0);

\node[linear] (llinG) at ($(L)+(18mm,-22mm)$) {Linear};
\node[linear] (llinF) at ($(llinG)+(-14mm,0)$) {Linear};
\node[linear] (llinQKVf) at ($(llinF)+(-18mm,0)$) {Linear};
\begin{scope}[on background layer]
    \node[linstack] at ($(llinQKVf)+(-3.2mm,3.2mm)$) {};
    \node[linstack] at ($(llinQKVf)+(-1.6mm,1.6mm)$) {};
\end{scope}

\node[conv, above=10mm of llinQKVf] (lconv) {Causal\\Conv1D};
\begin{scope}[on background layer]
    \node[convstack] at ($(lconv)+(-3.2mm,3.2mm)$) {};
    \node[convstack] at ($(lconv)+(-1.6mm,1.6mm)$) {};
\end{scope}

\node[actsilu] (laq) at ($(lconv.north)+(-7mm,11mm)$) {};
\node[actsilu] (lak) at ($(lconv.north)+(0,11mm)$) {};
\node[actsilu] (lav) at ($(lconv.north)+(7mm,11mm)$) {};
\node[actpsi] (laf) at ($(llinF.north)+(0,8mm)$) {$\psi$};
\node[actsilu] (lag) at ($(llinG.north)+(0,8mm)$) {};

\draw[arr] ($(lconv.north)+(-7mm,3.2mm)$) -- (laq);   
\draw[arr] ($(lconv.north)+(0,1.6mm)$)    -- (lak);   
\draw[arr] ($(lconv.north)+(7mm,0)$)      -- (lav);   

\draw[arr] ($(llinQKVf.north)+(-3.2mm,3.2mm)$) -- ($(lconv.south)+(-3.2mm,0)$);
\draw[arr] ($(llinQKVf.north)+(-1.6mm,1.6mm)$) -- ($(lconv.south)+(-1.6mm,0)$);
\draw[arr] (llinQKVf.north) -- (lconv.south);
\draw[arr] (llinF.north) -- (laf.south);
\draw[arr] (llinG.north) -- (lag.south);

\node[method, minimum width=30mm, minimum height=6mm] (lrec) at ($(lak)+(6mm,22mm)$) {M$^2$RNN recurrence};

\draw[arr] (laq.north) -- (laq.north |- lrec.south) node[lab,pos=0.65,right=0.4mm] {$q_t$};
\draw[arr] (lak.north) -- (lak.north |- lrec.south) node[lab,pos=0.65,right=0.4mm] {$k_t$};
\draw[arr] (lav.north) -- (lav.north |- lrec.south) node[lab,pos=0.65,right=0.4mm] {$v_t$};
\draw[arr] (laf.north) -- (laf.north |- lrec.south) node[lab,pos=0.55,right=0.4mm] {$f_t$};

\node[op, above=3mm of lrec] (lmul) {$\times$};
\node[norm, above=3mm of lmul] (lrms) {RMSNormTP};
\node[linear, above=3mm of lrms] (lfin) {Linear};

\draw[arr] (lrec.north) -- (lmul.south);
\draw[arr] (lmul.north) -- (lrms.south);
\draw[arr] (lrms.north) -- (lfin.south);

\draw[arr] (lag.north) |- (lmul.east)
    node[lab, pos=0.05, right=0.4mm] {$g_t$};

\coordinate (linBus) at ($(llinF.south)+(0,-4mm)$);
\coordinate (linBot) at ($(linBus)+(0,-5mm)$);

\draw[arr] (llinQKVf.south |- linBus) -- (llinQKVf.south);
\draw[arr] (linBus)                   -- (llinF.south);
\draw[arr] (llinG.south |- linBus)    -- (llinG.south);

\draw (llinQKVf.south |- linBus) -- (llinG.south |- linBus);

\begin{scope}[on background layer]
    \node[rounded corners=2pt, draw=gray!70, line width=0.6pt,
          fit=(llinQKVf) (llinG) (lfin) (lmul) (lrec) (linBus),
          inner xsep=5mm, inner ysep=3mm] (lpanel) {};
\end{scope}

\coordinate (R) at ($(lpanel.east|-L)+(35mm,0)$);

\node[linear] (rlinG) at ($(R)+(18mm,-22mm)$) {Linear};
\node[linear] (rlinF) at ($(rlinG)+(-14mm,0)$) {Linear};
\node[linear] (rlinQKVf) at ($(rlinF)+(-18mm,0)$) {Linear};
\begin{scope}[on background layer]
    \node[linstack] at ($(rlinQKVf)+(-3.2mm,3.2mm)$) {};
    \node[linstack] at ($(rlinQKVf)+(-1.6mm,1.6mm)$) {};
\end{scope}

\node[conv, above=10mm of rlinQKVf] (rconv) {Causal\\Conv1D};
\begin{scope}[on background layer]
    \node[convstack] at ($(rconv)+(-3.2mm,3.2mm)$) {};
    \node[convstack] at ($(rconv)+(-1.6mm,1.6mm)$) {};
\end{scope}

\node[actsilu] (raq) at ($(rconv.north)+(-7mm,11mm)$) {};
\node[actsilu] (rak) at ($(rconv.north)+(0,11mm)$) {};
\node[actsilu] (rav) at ($(rconv.north)+(7mm,11mm)$) {};
\node[actpsi] (raf) at ($(rlinF.north)+(0,8mm)$) {$\psi$};
\node[actsilu] (rag) at ($(rlinG.north)+(0,8mm)$) {};

\draw[arr] ($(rconv.north)+(-7mm,3.2mm)$) -- (raq);   
\draw[arr] ($(rconv.north)+(0,1.6mm)$)    -- (rak);   
\draw[arr] ($(rconv.north)+(7mm,0)$)      -- (rav);   

\draw[arr] ($(rlinQKVf.north)+(-3.2mm,3.2mm)$) -- ($(rconv.south)+(-3.2mm,0)$);
\draw[arr] ($(rlinQKVf.north)+(-1.6mm,1.6mm)$) -- ($(rconv.south)+(-1.6mm,0)$);
\draw[arr] (rlinQKVf.north) -- (rconv.south);
\draw[arr] (rlinF.north) -- (raf.south);
\draw[arr] (rlinG.north) -- (rag.south);

\node[method, minimum width=30mm, minimum height=6mm] (rrec) at ($(rak)+(6mm,22mm)$) {M$^2$RNN recurrence};

\draw[arr] (raq.north) -- (raq.north |- rrec.south) node[lab,pos=0.65,right=0.4mm] {$q_t$};
\draw[arr] (rak.north) -- (rak.north |- rrec.south) node[lab,pos=0.65,right=0.4mm] {$k_t$};
\draw[arr] (rav.north) -- (rav.north |- rrec.south) node[lab,pos=0.65,right=0.4mm] {$v_t$};
\draw[arr] (raf.north) -- (raf.north |- rrec.south) node[lab,pos=0.55,right=0.4mm] {$f_t$};

\node[op, above=3mm of rrec] (rmul) {$\times$};
\node[norm, above=3mm of rmul] (rrms) {RMSNormTP};
\node[linear, above=3mm of rrms] (rfin) {Linear};

\draw[arr] (rrec.north) -- (rmul.south);
\draw[arr] (rmul.north) -- (rrms.south);
\draw[arr] (rrms.north) -- (rfin.south);

\draw[arr] (rag.north) |- (rmul.east)
    node[lab, pos=0.05, right=0.4mm] {$g_t$};

\coordinate (rinBus) at ($(rlinF.south)+(0,-4mm)$);
\coordinate (rinBot) at ($(rinBus)+(0,-5mm)$);

\draw[arr] (rlinQKVf.south |- rinBus) -- (rlinQKVf.south);
\draw[arr] (rinBus)                   -- (rlinF.south);
\draw[arr] (rlinG.south |- rinBus)    -- (rlinG.south);

\draw (rlinQKVf.south |- rinBus) -- (rlinG.south |- rinBus);

\begin{scope}[on background layer]
    \node[rounded corners=2pt, draw=gray!70, line width=0.6pt,
          fit=(rlinQKVf) (rlinG) (rfin) (rmul) (rrec) (rinBus),
          inner xsep=5mm, inner ysep=3mm] (rpanel) {};
\end{scope}

\node[comms] (top) at ($(lfin)!0.5!(rfin)+(0,12mm)$) {ReduceScatter (FWD)\\AllGather (BWD)};
\draw[arr] (lfin.north) |- (top.west);
\draw[arr] (rfin.north) |- (top.east);
\draw[arr] (top.north) -- ++(0,6mm);

\node[comms] (bot) at ($(top|-linBus)+(0,-12mm)$) {AllGather (FWD)\\ReduceScatter (BWD)};
\draw[arr] ($(bot.south)+(0,-6mm)$) -- (bot.south);
\draw[arr] (bot.west) -| (linBus);
\draw[arr] (bot.east) -| (rinBus);

\draw[<->, dashed, green!55!black, line width=0.5pt] (lrms.east) -- (rrms.west)
    node[lab, midway, above, fill=white, inner sep=1pt] {AllReduce (FWD/BWD)};

\coordinate (lqkmid) at ($(lak.north)!0.3!(lak.north |- lrec.south)$);
\coordinate (rqkmid) at ($(rak.north)!0.3!(rak.north |- rrec.south)$);
\coordinate (lqkx)   at ($(laq)!0.5!(lak)$);
\coordinate (rqkx)   at ($(raq)!0.5!(rak)$);

\node[comms, minimum width=8mm, minimum height=4mm, font=\scriptsize, inner xsep=1pt]
    (lqkblock) at (lqkx |- lqkmid) {No-OP (FWD)\\AllReduce (BWD)};
\node[comms, minimum width=8mm, minimum height=4mm, font=\scriptsize, inner xsep=1pt]
    (rqkblock) at (rqkx |- rqkmid) {No-OP (FWD)\\AllReduce (BWD)};

\draw[<->, dashed, green!55!black, line width=0.5pt] (lqkblock.east) -- (rqkblock.west);

\end{tikzpicture}
    \caption{TP topology-independent \method~layer. Note that RMSNormTP have weights sharded along the model width dimension for the RMSNorm module on both GPUs requiring a synchronization in both the forward and backward computation.}
    \label{fig:tp-independent}
\end{figure*}

Alternatively, TP can be applied directly to the multi-value formulation ($N_q = N_k = 1$, $N_v = N$). Each GPU processes $\frac{N}{N_\textrm{TP}}$ value heads but shares the same query and key projections. This requires replacing RMSNorm \cite{rmsnorm} modules with RMSNormTP (refer to Appendix \ref{sec:rmsnormtp}) where the weight for RMSNorm is sharded and an extra communication is introduced in RMSNormTP for the forward and backward computation. The communication in the forward is required to synchronize the RMSNorm normalization term and the synchronization in the backward is required for computing the activation gradient. Note that this idea is applicable to \mamba~\cite{mamba2} and \gdn~\cite{gdn} layers as well. If the weights for the causal convolution and linear projections for queries and keys are initialized to the same values on all GPUs participating in TP, the sharding method illustrated in Figure \ref{fig:tp-independent} ensures that the weights remain synchronized during training.

\paragraph{Advantages.} This TP approach preserves the parameter count regardless of the number of GPUs and makes the \method~layers independent of the TP topology used during training.

\paragraph{Disadvantages.} This approach introduces additional AllReduce communications beyond standard TP. In the forward pass, one extra AllReduce is needed to synchronize the RMSNorm normalization term\footnote{The RMSNorm normalization term depends only on the batch size ($B$) and sequence length ($T$), so this communication is lightweight.}. The backward pass requires three extra AllReduces: one for the RMSNorm gradient and two to aggregate the query and key activation gradients, since each GPU computes partial gradients for the shared query and key heads from its assigned subset of value heads. Weight gradients for the query and key projections remain naturally synchronized across GPUs throughout training, provided the corresponding weights are initialized identically, so no additional communication is needed for them. One way to reduce backward communication is to make only the RMSNorm module TP topology-dependent (as in Section \ref{sec:tp-topology-aware}); since RMSNorm has very few parameters, this trade-off may be acceptable in practice.

\section{Experiments}
We train dense models with $410$M parameters (24 layers) and Mixture-of-Experts (MoE) \cite{shazeer2017outrageously} models with $7$B total ($1.1$B active) parameters (40 layers). All models are trained on 100B tokens sampled from the high quality subset of Nemotron-CC-v2 \cite{nemotron-cc-v2}.

\subsection{Models}
\subsubsection{Homogeneous Architectures}
We compare \method~primarily against \mamba~\cite{mamba2}, \gdn~\cite{gdn}, 
RNN and GRU \cite{gru}. We focus on \mamba~and \gdn~because they (and their variants) have demonstrated strong performance in state-of-the-art production models \cite{team2025kimi, nemotron-cc-v2, lieber2024jamba, ibm_granite_40_models}.

\subsubsection{2-way Hybrid Models}
We additionally run experiments in hybrid settings that combine recurrent layers with attention layers. Following common practice in production LLMs \cite{ibm_granite_40_models,lieber2024jamba}, we use 1 attention layer for every 7 recurrent layers (i.e., 1 out of every 8 layers is attention). We compare Hybrid \method~against Transformer++ \cite{attention}, Hybrid \mamba~and Hybrid \gdn.

\subsubsection{3-way Hybrid Models}
We also experiment with augmenting Hybrid \gdn~models with \method~layers, yielding 3-way hybrids denoted ``Hybrid GDN + \method-$n$'', where $n$ is the number of \method~layers. For $n = 1$, we replace the first recurrent layer with \method; for $n > 1$, we replace the layer immediately preceding each attention layer, producing ``Hybrid GDN + \method-3'' and ``Hybrid GDN + \method-5'' for the 410M and 7B models respectively.

\subsection{Model Training}
For all the models, we use RMSNorm \cite{rmsnorm} with pre-normalization \cite{prenorm}. Additionally, Transformer++ uses RoPE positional embeddings \cite{su2024roformer} while all linear and hybrid architectures use NoPE (no positional embeddings) \cite{kazemnejad2023impact}. For MLP and MoE layers \cite{shazeer2017outrageously}, we use the SwiGLU activation \cite{shazeer2020glu}. We keep the model width, intermediate MLP width, number of layers, and all other hyperparameters consistent across architectures, adjusting only the parameters in the sequence mixing block.

All models are trained with the AdamW optimizer \cite{loshchilov2019decoupledweightdecayregularization} using a cosine learning rate schedule. We linearly warm up the learning rate over the first 10B tokens to a peak value of $3 \times 10^{-4}$, followed by decaying to $10\%$ of the peak ($3 \times 10^{-5}$). We use gradient clipping of $1$ and weight decay of $0.1$. We don't apply weight decay to the normalization parameters. All models are trained with a batch size of $1024$ samples at a context length of $4096$ tokens. We ensure that all models see identical data in identical order to ensure fair comparison. The 410M dense models are trained on 32 H100s and the 7B MoE models on 64 H100s, with each experiment taking between 14--20 hours.

We use the lm-engine codebase \cite{mishra2024lmengine} written in PyTorch \cite{paszke2019pytorch} for training all our models. We also use FlashAttention-3 \cite{shah2024flashattention} and SonicMoE \cite{guo2025sonicmoe} with PyTorch compile \cite{pytorch2} for training throughput optimization. We also use fast kernels for causal convolutions, \mamba, \gdn, and \method~where applicable. The models use mixed precision \cite{micikevicius2017mixed} in BF16 \cite{kalamkar2019study}, with gradient AllReduce and accumulation in FP32 for numerical stability.

\subsection{Language Modeling}

We evaluate all trained models on a suite of downstream benchmarks using EleutherAI's LM Eval Harness \cite{eval-harness}: LAMBADA \cite{paperno2016lambada}, HellaSwag \cite{zellers2019hellaswagmachinereallyfinish}, PIQA \cite{bisk2019piqareasoningphysicalcommonsense}, ARC \cite{arc}, Winogrande \cite{sakaguchi2019winograndeadversarialwinogradschema}, BoolQ \cite{clark2019boolqexploringsurprisingdifficulty}, OpenBookQA \cite{obqa}, CoPA \cite{copa}, and SciQ \cite{sciq}. We additionally report language modeling perplexity on Wikitext \cite{wikitext} and LAMBADA \cite{paperno2016lambada}.

\begin{table}[htbp]
\centering
\resizebox{\textwidth}{!}{%
\begin{tabular}{lccccccccccccc}
\toprule
Model & Wiki PPL ($\downarrow$) & LMB PPL ($\downarrow$) & LAMBADA & HellaSwag & PIQA & ARC-E & ARC-C & WinoGrande & BoolQ & OBQA & COPA & SciQ & Avg Acc \\
\midrule
\mamba & 22.93 & 38.56 & 32.78 & 47.70 & 70.29 & 62.04 & 30.55 & 53.28 & 58.69 & 33.20 & 69.00 & 85.50 & \underline{54.30} \\
\gdn & \underline{22.70} & 34.17 & 34.21 & 47.16 & 71.00 & 61.24 & 28.07 & 51.07 & 58.84 & 32.80 & 70.00 & 85.50 & 53.99 \\
RNN & 33.74 & 317.38 & 13.00 & 44.12 & 69.70 & 62.54 & 30.63 & 49.72 & 45.75 & 33.40 & 63.00 & 77.10 & 48.90 \\
GRU & 25.80 & 63.83 & 26.10 & 48.70 & 71.49 & 64.90 & 30.97 & 49.64 & 57.68 & 33.80 & 71.00 & 82.90 & 53.72 \\
\method & 22.92 & \underline{33.63} & 33.98 & 47.76 & 71.60 & 62.29 & 30.29 & 53.43 & 50.21 & 34.80 & 72.00 & 86.00 & 54.24 \\
\midrule
Transformer++ & 23.32 & 42.87 & 33.36 & 45.02 & 68.88 & 59.76 & 29.01 & 50.43 & 54.56 & 32.00 & 72.00 & 84.00 & 52.90 \\
Hybrid \mamba & 21.59 & 35.84 & 35.32 & 48.11 & 70.89 & 62.54 & 29.35 & 52.72 & 51.41 & 35.40 & 65.00 & 86.80 & 53.75 \\
Hybrid \gdn & 21.89 & \underline{29.02} & 35.78 & 47.19 & 70.67 & 61.91 & 29.78 & 52.17 & 52.57 & 34.00 & 70.00 & 86.30 & 54.04 \\
Hybrid \method & \underline{21.53} & 37.31 & 35.67 & 48.74 & 70.29 & 63.38 & 30.29 & 51.46 & 50.09 & 33.80 & 71.00 & 86.80 & \underline{54.15} \\
\midrule
Hybrid \mamba & 21.59 & 35.84 & 35.32 & 48.11 & 70.89 & 62.54 & 29.35 & 52.72 & 51.41 & 35.40 & 65.00 & 86.80 & 53.75 \\
Hybrid \mamba~+ \method-1 & 21.48 & 36.98 & 34.47 & 48.23 & 70.46 & 64.14 & 31.57 & 54.46 & 49.54 & 33.60 & 70.00 & 85.80 & \underline{54.23} \\
Hybrid \mamba~+ \method-3 & \underline{21.39} & \underline{35.77} & 35.77 & 48.47 & 70.57 & 62.33 & 31.83 & 50.75 & 51.44 & 33.00 & 65.00 & 87.40 & 53.65 \\
\midrule
Hybrid \gdn & 21.89 & 29.02 & 35.78 & 47.19 & 70.67 & 61.91 & 29.78 & 52.17 & 52.57 & 34.00 & 70.00 & 86.30 & 54.04 \\
Hybrid \gdn~+ \method-1 & 21.39 & 33.78 & 34.83 & 48.11 & 70.73 & 61.24 & 29.95 & 51.70 & 48.41 & 33.20 & 72.00 & 87.30 & 53.75 \\
Hybrid \gdn~+ \method-3 & \textbf{21.26} & \textbf{28.93} & 37.18 & 48.83 & 69.64 & 63.85 & 29.44 & 53.35 & 59.14 & 33.80 & 71.00 & 86.10 & \textbf{55.23} \\
\bottomrule
\end{tabular}
}
\caption{Language Modeling on the 410M parameter dense model across commonsense reasoning benchmarks. All models are evaluated in the 0-shot setting. Accuracies ($\uparrow$) are higher-is-better.}
\label{table:410m}
\end{table}

\begin{table}[htbp]
\centering
\resizebox{\tablewidth}{!}{%
\begin{tabular}{lccccccccccccc}
\toprule
Model & Wiki PPL ($\downarrow$) & LMB PPL ($\downarrow$) & LAMBADA & HellaSwag & PIQA & ARC-E & ARC-C & WinoGrande & BoolQ & OBQA & COPA & SciQ & Avg Acc \\
\midrule
\mamba & \underline{13.73} & \underline{11.03} & 50.09 & 67.43 & 76.82 & 74.33 & 43.52 & 61.64 & 64.01 & 41.00 & 80.00 & 93.40 & \underline{65.22} \\
\gdn & 13.89 & 11.10 & 48.92 & 66.94 & 77.26 & 73.82 & 40.53 & 60.14 & 65.41 & 39.60 & 78.00 & 92.20 & 64.28 \\
RNN & 17.65 & 43.68 & 30.25 & 64.49 & 76.61 & 72.64 & 41.04 & 55.64 & 58.35 & 40.20 & 79.00 & 87.60 & 60.58 \\
GRU & 14.80 & 15.33 & 43.94 & 67.36 & 76.82 & 75.67 & 43.52 & 58.96 & 63.79 & 41.00 & 82.00 & 91.30 & 64.44 \\
\method & 13.80 & 11.48 & 49.18 & 67.57 & 76.71 & 75.51 & 44.97 & 58.17 & 64.16 & 41.40 & 82.00 & 92.00 & 65.17 \\
\midrule
Transformer++ & 14.94 & 17.26 & 43.90 & 62.99 & 75.63 & 71.97 & 38.91 & 57.54 & 63.12 & 37.20 & 75.00 & 91.40 & 61.76 \\
Hybrid \mamba & 13.10 & 11.05 & 50.18 & 68.62 & 76.55 & 72.64 & 43.09 & 60.62 & 63.61 & 43.00 & 79.00 & 93.00 & 65.03 \\
Hybrid \gdn & 13.51 & 11.06 & 50.20 & 67.23 & 77.09 & 73.06 & 41.64 & 59.12 & 62.39 & 40.80 & 82.00 & 92.90 & 64.64 \\
Hybrid \method & \underline{13.00} & \underline{10.58} & 51.74 & 68.44 & 76.55 & 76.81 & 43.09 & 61.09 & 63.21 & 40.40 & 81.00 & 93.30 & \underline{65.56} \\
\midrule
Hybrid \mamba & 13.10 & 11.05 & 50.18 & 68.62 & 76.55 & 72.64 & 43.09 & 60.62 & 63.61 & 43.00 & 79.00 & 93.00 & 65.03 \\
Hybrid \mamba~+ \method-1 & \underline{13.01} & \underline{10.84} & 50.92 & 68.35 & 77.86 & 74.66 & 42.24 & 60.54 & 62.57 & 40.40 & 83.00 & 92.80 & 65.33 \\
Hybrid \mamba~+ \method-5 & 13.02 & 11.59 & 49.84 & 68.19 & 77.04 & 75.17 & 42.58 & 59.04 & 65.26 & 42.40 & 83.00 & 93.90 & \underline{65.64} \\
\midrule
Hybrid \gdn & 13.51 & 11.06 & 50.20 & 67.23 & 77.09 & 73.06 & 41.64 & 59.12 & 62.39 & 40.80 & 82.00 & 92.90 & 64.64 \\
Hybrid \gdn~+ \method-1 & 13.07 & 10.33 & 51.54 & 68.69 & 78.18 & 74.45 & 41.81 & 60.30 & 65.08 & 42.00 & 82.00 & 93.90 & \textbf{65.80} \\
Hybrid \gdn~+ \method-5 & \textbf{12.85} & \textbf{10.29} & 51.80 & 68.99 & 76.82 & 76.73 & 43.60 & 60.69 & 63.30 & 40.20 & 81.00 & 93.50 & 65.66 \\
\bottomrule
\end{tabular}
}
\caption{Language Modeling on the 7B (1B active) MoE model across commonsense reasoning benchmarks. All models are evaluated in the 0-shot setting. Accuracies ($\uparrow$) are higher-is-better.}
\label{table:7b}
\tablepostskip
\end{table}

\subsubsection{Homogeneous Architectures}
At the 410M scale, \gdn~achieves the best Wikitext \cite{wikitext} perplexity among all baselines, with \method~matching \mamba~\cite{mamba2} (within $0.01$ points). On LAMBADA \cite{paperno2016lambada}, \method~outperforms both \mamba~and \gdn~\cite{gdn} by $4.93$ and $0.54$ perplexity points respectively (Table \ref{table:410m}). We attribute \method's gap behind \gdn~on Wikitext to its $3\times$ smaller state size: while increasing the state size can close this gap, it also raises the per-step FLOP count of the recurrence, making it a less favorable tradeoff in practice.

At the 7B MoE scale, \mamba~achieves the best Wikitext perplexity, followed by \method, which outperforms \gdn~by $0.09$ points (Table \ref{table:7b}). On average downstream accuracy, \method~matches \mamba~to within $0.06$ points at both the 410M and 7B MoE scales while consistently outperforming \gdn~(Tables \ref{table:410m}, \ref{table:7b}). Vanilla RNN and GRU models significantly underperform \method~and other linear RNNs.

\subsubsection{2-way Hybrid Models}
Combining \method~with attention \cite{attention} (denoted Hybrid \method) yields substantial gains on Wikitext \cite{wikitext}. Hybrid \method~outperforms Hybrid \mamba~\cite{mamba2} by $0.06$ and $0.1$ perplexity points at 410M and 7B MoE scale, and Hybrid \gdn~\cite{gdn} by $0.4$ and $0.5$ points at the 410M and 7B MoE scale. Average downstream accuracy also improves over all other hybrids at both scales (Tables \ref{table:410m}, \ref{table:7b}). Overall, 2-way hybrids improve upon both Transformer++ and homogeneous \method~models. We attribute these gains to \method's non-linear state transition: in the homogeneous setting, \method's advantage is constrained by its smaller state size relative to \mamba~and \gdn, however, in the hybrid setting, attention is responsible for in-context retrieval, allowing \method's non-linear recurrence to contribute expressivity that linear RNNs such as \mamba~and \gdn~cannot provide.

\subsubsection{3-way Hybrid Models}
We next explore 3-way hybrids that mix attention, \gdn~(or \mamba), and \method~layers. We find that replacing even a single linear RNN layer with \method~(denoted Hybrid \gdn~(or \mamba) + \method-1) suffices to match the accuracy of the full Hybrid \method~model while preserving comparable training throughput (Section \ref{section:throughput}). This result indicates that \method~layers carry information that is substantially different from what \gdn~layers capture.

Hybrid \gdn~+ \method-1 improves upon Hybrid \gdn, reducing Wikitext perplexity by $0.5$ and $0.44$ points at the 410M and 7B (1B active) scales, respectively. We also observe improvements on Hybrid \mamba~+ \method-1 over Hybrid \mamba~of $0.1$ perplexity points on both model scales. Increasing the number of \method~layers to 3 (5 for 7B MoE) yields additional gains of $0.13$ and $0.22$ perplexity points. Together, these results solidify the importance of non-linear recurrences in sequence mixing blocks.

\subsection{In-Context Retrieval}
\subsubsection{In-Context Retrieval on RULER}
\label{section:ruler}
We evaluate the in-context retrieval ability of \method~relative to other linear RNNs on the RULER benchmark \cite{hsieh2024ruler}, focusing on the Single Needle In A Haystack (S-NIAH-1, 2, 3) \cite{nelson2024needle}, Multi-Query (MQ), Multi-Key (MK), and Multi-Value (MV) variants:

\begin{enumerate}
    \item \textbf{S-NIAH}: A key-value pair is embedded in context and the model is tasked to retrieve the value given the key. S-NIAH-1 uses a synthetic passage with word keys and numeric values, probing long-term retention. S-NIAH-2 and S-NIAH-3 use real-world essays with numeric and UUID values, respectively, testing efficient long-context memory management \cite{gdn}.

    \item \textbf{MQ-NIAH}: In the MQ-NIAH benchmark, multiple key--value pairs are embedded within the passage, and the task requires retrieving all needles associated with distinct keys.

    \item \textbf{MK-NIAH}: The MK-NIAH benchmark inserts multiple key-value pairs in the passage. The task is to retrieve a single specified key while the other keys serve as distractors.

    \item \textbf{MV-NIAH}: Multiple values share the same key and the model must retrieve all values associated with that key.
\end{enumerate}

Figures \ref{fig:niah-410m-gdn} and \ref{fig:niah-410m-mamba} report accuracy at the 410M scale for \gdn~and \mamba-based models, respectively; Figures \ref{fig:niah-7b-gdn} and \ref{fig:niah-7b-mamba} report the same for the 7B (1B active) MoE model. All models are trained with a 4,096 context length and evaluated up to 16,384.

\paragraph{Hybrid \gdn~+ \method:} Replacing a few \gdn~layers with \method~layers in the hybrid setting significantly improves retrieval at unseen context lengths on S-NIAH-2, and MQ/MK/MV-NIAH at the 410M parameter scale (Figure \ref{fig:niah-410m-gdn}). At the 7B MoE scale, we observe that the 3-way consistently improves over the 2-way hybrid models except on S-NIAH-1 (Figure \ref{fig:niah-7b-gdn}). Notably, augmenting with \method~layers also improves MQ/MK/MV-NIAH accuracy at context lengths seen during training ($\le 4,096$) at both model scales.

On S-NIAH-1, we find that both Hybrid \method~and Hybrid \gdn~achieve 100\% accuracy across all context lengths, however, adding \method~layers leads to accuracy degradation, likely due to the extremely synthetic passages used in the S-NIAH-1 task.

\paragraph{Hybrid \mamba~+ \method:} We find that a single \method~layer with Hybrid \mamba~(Hybrid \mamba~+ \method-1) is insufficient to improve long-context retrieval at either model scale. However, replacing three \mamba~layers with \method~layers at 410M and 5 layers at 7B MoE yields consistent improvement (Figures \ref{fig:niah-410m-mamba}, \ref{fig:niah-7b-mamba}). We attribute this to the weaker state transition of \mamba~requiring more \method~layers to compensate.

\begin{figure}[h!]
    \centering
    \includegraphics[width=\textwidth]{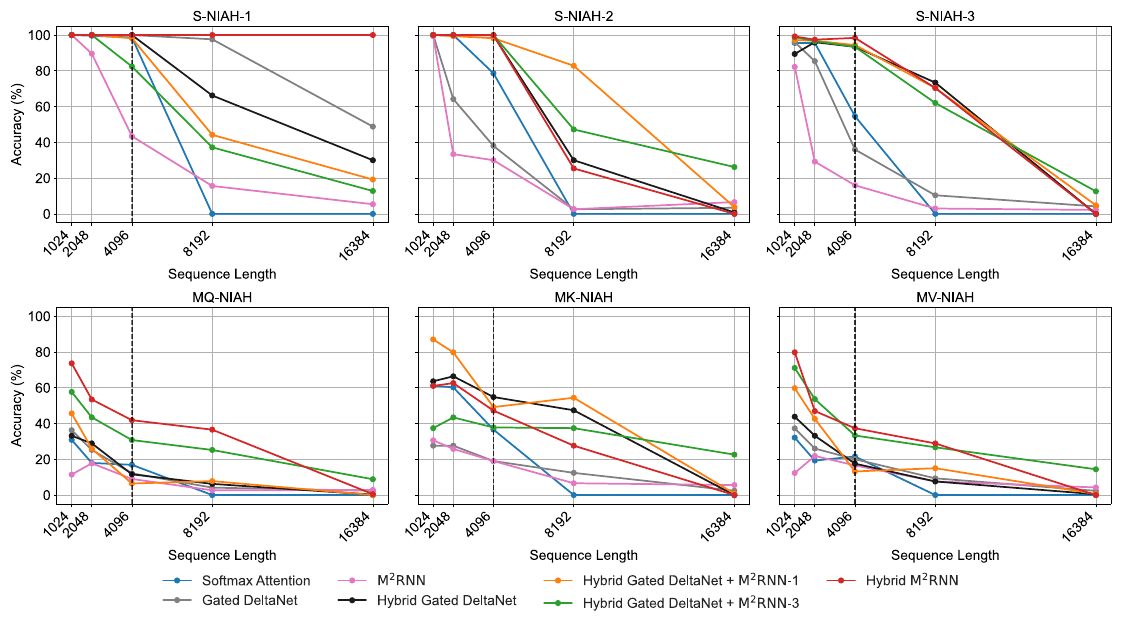}
    \caption{Zero-shot in-context retrieval performance on the RULER benchmark for \method~layers and \gdn~models on 410M parameter scale. The vertical lines at 4096 indicate the training context length.}
    \label{fig:niah-410m-gdn}

    \centering
    \includegraphics[width=\textwidth]{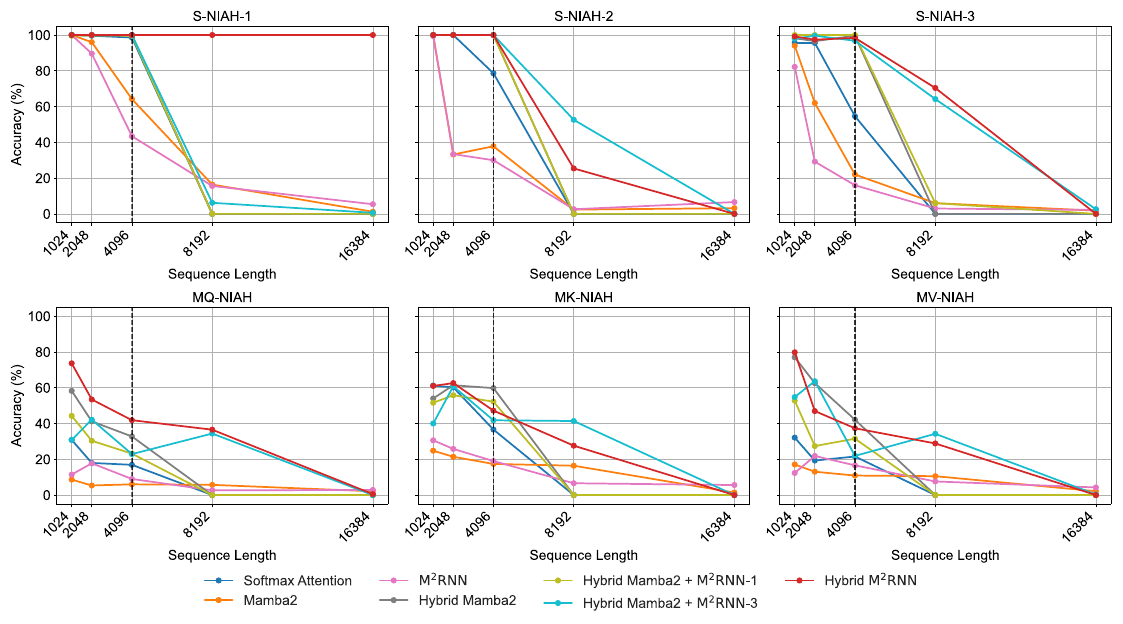}
    \caption{Zero-shot in-context retrieval performance on the RULER benchmark for \method~layers and \mamba~models on 410M parameter scale. The vertical lines at 4096 indicate the training context length.}
    \label{fig:niah-410m-mamba}
\end{figure}

\begin{figure}[h!]
    \centering
    \includegraphics[width=\textwidth]{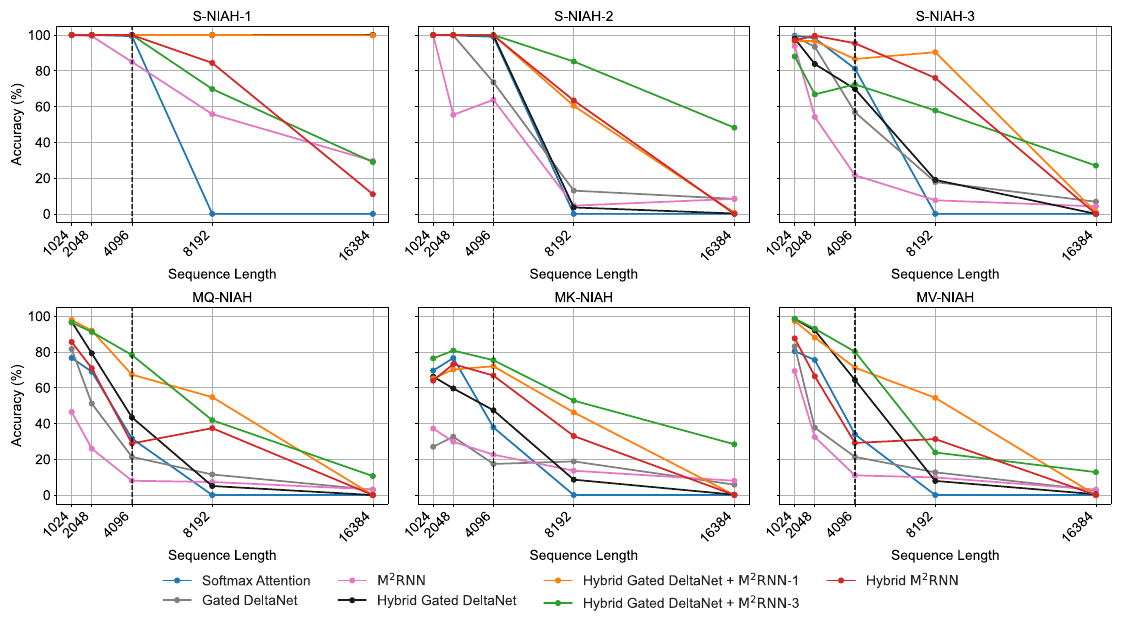}
    \caption{Zero-shot in-context retrieval performance on the RULER benchmark for \method~layers and \gdn~models for the 7B MoE model. The vertical lines at 4096 indicate the training context length.}
    \label{fig:niah-7b-gdn}

    \centering
    \includegraphics[width=\textwidth]{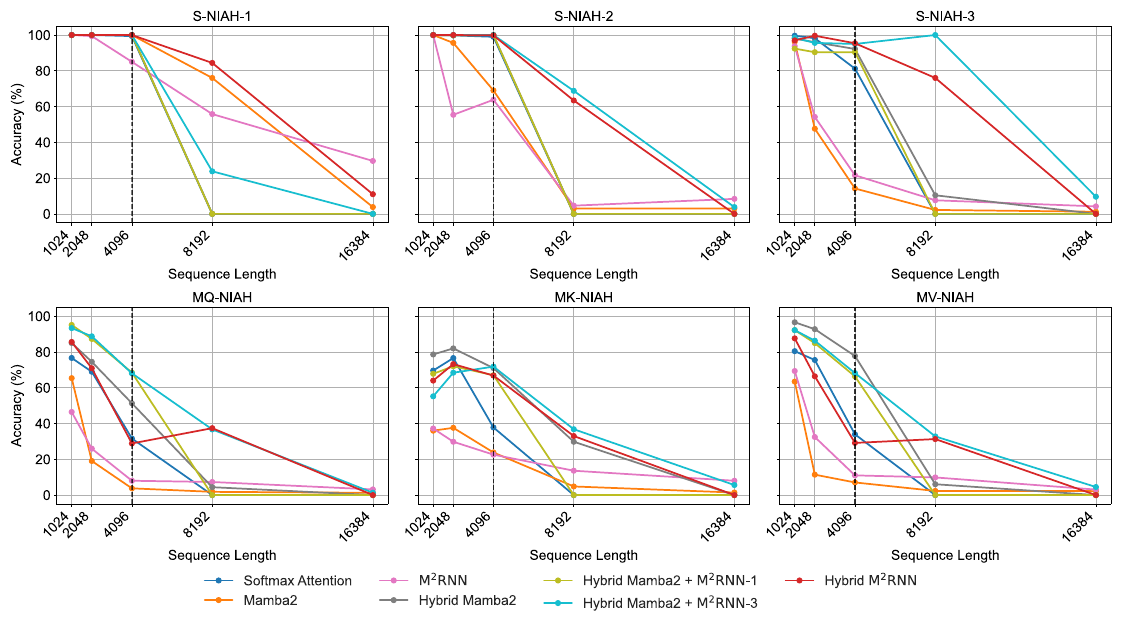}
    \caption{Zero-shot in-context retrieval performance on the RULER benchmark for \method~layers and \mamba~models for the 7B MoE model. The vertical lines at 4096 indicate the training context length.}
    \label{fig:niah-7b-mamba}
\end{figure}

\begin{table}[h!]
\centering
\scriptsize
\resizebox{\textwidth}{!}{%
\begin{tabular}{lccccccc}
\toprule
Model & SQuAD & NQ & DROP & TriviaQA & FDA & SWDE & Avg \\
\midrule
\mamba & 31.2 & 15.0 & 19.5 & 50.5 & 15.6 & 32.4 & 27.4 \\
\gdn & 34.8 & 15.4 & 24.9 & 51.7 & 16.3 & 41.8 & \underline{30.8} \\
RNN & 9.6 & 5.9 & 14.6 & 29.7 & 0.0 & 5.0 & 10.8 \\
GRU & 15.8 & 8.8 & 17.2 & 41.5 & 0.7 & 7.7 & 15.3 \\
\method & 34.0 & 13.0 & 22.9 & 52.3 & 10.6 & 28.5 & 26.9 \\
\midrule
Transformer++ & 14.0 & 15.5 & 28.2 & 52.4 & 46.0 & 51.7 & 34.6 \\
Hybrid \mamba & 37.6 & 21.4 & 26.4 & 51.7 & 65.3 & 61.7 & 44.0 \\
Hybrid \gdn & 37.6 & 22.2 & 26.0 & 53.4 & 61.5 & 58.0 & 43.1 \\
Hybrid \method & 41.3 & 22.7 & 26.8 & 55.5 & 74.5 & 60.7 & \textbf{46.9} \\
\midrule
Hybrid \mamba & 37.6 & 21.4 & 26.4 & 51.7 & 65.3 & 61.7 & 44.0 \\
Hybrid \mamba~+ \method-1 & 38.6 & 23.8 & 26.3 & 55.7 & 61.3 & 60.2 & 44.3 \\
Hybrid \mamba~+ \method-3 & 40.4 & 23.5 & 26.4 & 56.7 & 65.2 & 61.7 & \underline{45.6} \\
\midrule
Hybrid \gdn & 37.6 & 22.2 & 26.0 & 53.4 & 61.5 & 58.0 & 43.1 \\
Hybrid \gdn~+ \method-1 & 39.8 & 23.5 & 25.6 & 54.6 & 67.2 & 62.5 & \underline{45.5} \\
Hybrid \gdn~+ \method-3 & 39.3 & 21.4 & 27.2 & 54.9 & 65.1 & 56.5 & 44.1 \\
\bottomrule
\end{tabular}
}
\caption{In-context retrieval performance ($\uparrow$) on real-world data at the 410M parameter scale.}
\label{table:410m-real-retrieval}
\end{table}

\begin{table}[h!]
\centering
\scriptsize
\resizebox{\tablewidth}{!}{%
\begin{tabular}{lccccccc}
\toprule
Model & SQuAD & NQ & DROP & TriviaQA & FDA & SWDE & Avg \\
\midrule
\mamba & 41.7 & 26.5 & 30.0 & 65.0 & 33.0 & 54.4 & 41.8 \\
\gdn & 40.3 & 24.2 & 29.9 & 64.3 & 32.4 & 61.8 & \underline{42.2} \\
RNN & 20.0 & 12.2 & 20.4 & 53.1 & 0.8 & 9.8 & 19.4 \\
GRU & 27.0 & 16.7 & 24.6 & 59.3 & 5.2 & 18.4 & 25.2 \\
\method & 37.9 & 21.1 & 28.2 & 62.5 & 27.6 & 49.9 & 37.9 \\
\midrule
Transformer++ & 22.1 & 24.0 & 27.6 & 62.2 & 67.0 & 70.1 & 45.5 \\
Hybrid \mamba & 48.2 & 31.4 & 32.7 & 67.9 & 73.1 & 77.1 & 55.1 \\
Hybrid \gdn & 46.4 & 31.2 & 28.3 & 63.6 & 67.1 & 71.9 & 51.4 \\
Hybrid \method & 48.4 & 31.4 & 31.4 & 64.6 & 78.0 & 79.8 & \textbf{55.6} \\
\midrule
Hybrid \mamba & 48.2 & 31.4 & 32.7 & 67.9 & 73.1 & 77.1 & 55.1 \\
Hybrid \mamba~+ \method-1 & 46.6 & 30.6 & 33.6 & 67.1 & 78.4 & 75.0 & \underline{55.2} \\
Hybrid \mamba~+ \method-5 & 46.8 & 31.5 & 28.7 & 66.0 & 78.5 & 79.7 & \underline{55.2} \\
\midrule
Hybrid \gdn & 46.4 & 31.2 & 28.3 & 63.6 & 67.1 & 71.9 & 51.4 \\
Hybrid \gdn~+ \method-1 & 46.4 & 29.6 & 32.1 & 66.4 & 73.1 & 77.4 & 54.2 \\
Hybrid \gdn~+ \method-5 & 47.1 & 32.0 & 31.2 & 65.5 & 77.2 & 77.9 & \underline{55.1} \\
\bottomrule
\end{tabular}
}
\caption{In-context retrieval performance ($\uparrow$) on real-world data at the 7B (1B active) MoE scale.}
\label{table:7b-real-retrieval}
\tablepostskip
\end{table}

\subsubsection{In-Context Retrieval on Real-World Data}
\label{section:in-context-retrieval-real}
\citet{deltanet} show that despite strong synthetic-task performance, DeltaNet still underperforms \mamba~\cite{mamba2} on real-world retrieval. To assess whether this limitation persists for \method, we evaluate in-context retrieval on real-world data for both model scales (Tables \ref{table:410m-real-retrieval}, \ref{table:7b-real-retrieval}).

As expected, Transformer++ \cite{attention} substantially outperforms all recurrent models due to its linearly growing key-value cache, which enables direct retrieval from any past position. Among purely recurrent architectures, \gdn~\cite{gdn} outperforms \mamba~\cite{mamba2} and \method~which we attribute to the $3 \times$ larger state size of \gdn. In the hybrid setting, all models significantly improve upon the Transformer++ baseline, with Hybrid \method~achieving the largest gains: $12.3$ and $10.1$ point gains across the 410M and 7B MoE model scales, outperforming both Hybrid \mamba~and Hybrid \gdn. Further, we find that adding \method~layers to Hybrid \gdn~(Hybrid \gdn~+ \method) further improves upon Hybrid \gdn. We observe similar improvements for Hybrid \mamba~when combined with \method~(Tables \ref{table:410m-real-retrieval}, \ref{table:7b-real-retrieval}).

\subsection{Long-Context Performance}
\begin{table}[htbp]
\centering
\resizebox{\textwidth}{!}{%
\begin{tabular}{lccccccccc}
\toprule
Model & GovReport & QMSum & MultiNews & TREC & TriviaQA & SAMSum & LCC & RepoBench-P & Avg \\
\midrule
\mamba & 6.6 & 14.3 & 9.9 & 26.5 & 23.2 & 11.5 & 13.7 & 8.9 & 14.3 \\
\gdn & 7.6 & 18.2 & 12.0 & 19.0 & 34.1 & 21.3 & 14.9 & 11.8 & \underline{17.3} \\
RNN & 9.3 & 15.2 & 9.2 & 0.5 & 12.2 & 5.0 & 17.9 & 14.9 & 10.5 \\
GRU & 7.0 & 16.5 & 10.8 & 1.0 & 21.4 & 8.5 & 12.5 & 9.0 & 10.8 \\
\method & 6.3 & 16.4 & 10.8 & 22.0 & 24.3 & 13.2 & 10.7 & 10.1 & 14.2 \\
\midrule
Transformer++ & 5.6 & 5.7 & 6.2 & 11.0 & 9.7 & 4.6 & 12.0 & 9.9 & 8.1 \\
Hybrid \mamba & 10.0 & 4.8 & 10.2 & 16.0 & 9.4 & 9.2 & 10.5 & 12.0 & 10.3 \\
Hybrid \gdn & 8.2 & 16.8 & 16.1 & 19.5 & 28.7 & 15.9 & 8.7 & 9.2 & \underline{15.4} \\
Hybrid \method & 10.0 & 14.0 & 7.2 & 17.5 & 35.5 & 14.9 & 13.4 & 10.3 & \underline{15.4} \\
\midrule
Hybrid \mamba & 10.0 & 4.8 & 10.2 & 16.0 & 9.4 & 9.2 & 10.5 & 12.0 & 10.3 \\
Hybrid \mamba~+ \method-1 & 9.8 & 9.0 & 7.5 & 22.5 & 13.3 & 12.6 & 13.2 & 9.8 & 12.2 \\
Hybrid \mamba~+ \method-3 & 10.5 & 15.8 & 13.2 & 29.0 & 31.7 & 12.1 & 10.7 & 13.3 & \underline{17.1} \\
\midrule
Hybrid \gdn & 8.2 & 16.8 & 16.1 & 19.5 & 28.7 & 15.9 & 8.7 & 9.2 & 15.4 \\
Hybrid \gdn~+ \method-1 & 9.5 & 14.8 & 17.6 & 47.0 & 32.1 & 19.4 & 7.3 & 8.5 & 19.5 \\
Hybrid \gdn~+ \method-3 & 16.0 & 18.1 & 14.0 & 43.0 & 41.5 & 24.7 & 12.7 & 13.9 & \textbf{23.0} \\
\bottomrule
\end{tabular}
}
\caption{Performance ($\uparrow$) comparison on the 410M model across long-context summarization, coding and few-shot learning benchmarks.}
\label{table:410m-longbench}
\end{table}

We evaluate long-context performance on LongBench \cite{bai2024longbench} (Table \ref{table:410m-longbench}, \ref{table:7b-longbench}). In the 2-way hybrid setting, we find that Hybrid \method~matches Hybrid \gdn~at 410M parameter scale and outperforms Hybrid \gdn~by $1.5$ points on the 7B MoE scale. We find that combining \method~layer with Hybrid \gdn~results in up to 8 points of average accuracy improvement across both model scales across long-context summarization, coding, and few-shot learning tasks. Similar improvements are observable with Hybrid \mamba~+ \method-3 (5 for 7B MoE) models on both model scales.
\begin{table}[htbp]
\centering
\resizebox{\tablewidth}{!}{%
\begin{tabular}{lccccccccc}
\toprule
Model & GovReport & QMSum & MultiNews & TREC & TriviaQA & SAMSum & LCC & RepoBench-P & Avg \\
\midrule
\mamba & 5.1 & 16.0 & 13.4 & 42.0 & 59.7 & 22.0 & 13.4 & 8.7 & 22.5 \\
\gdn & 4.8 & 20.3 & 12.2 & 47.0 & 61.6 & 36.2 & 9.7 & 7.9 & \underline{25.0} \\
RNN & 9.3 & 16.7 & 10.5 & 4.0 & 27.5 & 13.0 & 17.1 & 15.8 & 14.2 \\
GRU & 5.7 & 18.0 & 10.8 & 9.0 & 52.4 & 24.6 & 20.9 & 13.9 & 19.4 \\
\method & 8.1 & 19.1 & 11.8 & 26.5 & 51.3 & 30.7 & 20.4 & 19.3 & 23.4 \\
\midrule
Transformer++ & 6.1 & 6.1 & 17.0 & 16.0 & 11.4 & 10.0 & 13.1 & 10.2 & 11.2 \\
Hybrid \mamba & 9.1 & 7.7 & 20.3 & 62.5 & 23.3 & 12.1 & 5.6 & 9.8 & 18.8 \\
Hybrid \gdn & 10.4 & 15.8 & 16.7 & 57.5 & 35.1 & 23.6 & 9.7 & 16.3 & 23.1 \\
Hybrid \method & 10.9 & 13.9 & 21.1 & 52.0 & 44.7 & 18.5 & 19.8 & 15.7 & \underline{24.6} \\
\midrule
Hybrid \mamba & 9.1 & 7.7 & 20.3 & 62.5 & 23.3 & 12.1 & 5.6 & 9.8 & 18.8 \\
Hybrid \mamba~+ \method-1 & 8.6 & 3.9 & 21.2 & 25.5 & 15.0 & 11.5 & 7.5 & 11.6 & 13.1 \\
Hybrid \mamba~+ \method-5 & 18.6 & 10.3 & 7.1 & 60.5 & 45.7 & 21.4 & 9.7 & 8.0 & \underline{22.7} \\
\midrule
Hybrid \gdn & 10.4 & 15.8 & 16.7 & 57.5 & 35.1 & 23.6 & 9.7 & 16.3 & 23.1 \\
Hybrid \gdn~+ \method-1 & 17.2 & 16.5 & 6.0 & 69.5 & 64.9 & 34.1 & 20.1 & 15.1 & 30.4 \\
Hybrid \gdn~+ \method-5 & 12.2 & 16.0 & 16.1 & 59.0 & 73.5 & 36.0 & 18.0 & 18.9 & \textbf{31.2} \\
\bottomrule
\end{tabular}
}
\caption{LongBench results ($\uparrow$) on the 7B (1B active) MoE model.}
\label{table:7b-longbench}
\tablepostskip
\end{table}

\paragraph{BabiLong.} Table \ref{table:babilong} reports results on the BabiLong benchmark \cite{kuratov2024babilongtestinglimitsllmstestinglimitsllms} for the 7B MoE models. Hybrid \method~achieves the highest average accuracy, outperforming Hybrid \gdn~by 3.9 points and Transformer++ by 2.5 points.

\begin{table}[h!]
\centering
\scriptsize
\resizebox{\linewidth}{!}{%
\begin{tabular}{lcccccc}
\toprule
Task & Transformer++ & \mamba & \gdn & \method & Hybrid \gdn & Hybrid \method \\
\midrule
Single supporting fact QA & 7.30 & 2.00 & 3.90 & 4.20 & 4.80 & 14.00 \\
Two supporting facts QA & 2.00 & 1.80 & 4.90 & 1.10 & 6.91 & 14.31 \\
Three supporting facts QA & 7.11 & 4.80 & 9.51 & 7.61 & 19.22 & 11.61 \\
Two argument relations & 3.00 & 2.50 & 5.21 & 2.30 & 15.62 & 24.02 \\
Three argument relations & 18.12 & 11.61 & 22.42 & 22.42 & 28.73 & 32.43 \\
Yes/No questions & 49.70 & 55.80 & 52.20 & 47.00 & 58.50 & 46.30 \\
Counting & 48.80 & 49.50 & 95.40 & 48.80 & 95.60 & 95.20 \\
Lists and sets & 51.20 & 48.30 & 54.20 & 46.50 & 41.60 & 53.90 \\
Simple negation & 96.90 & 64.30 & 63.80 & 64.10 & 47.70 & 66.50 \\
Indefinite knowledge & 46.60 & 43.90 & 43.90 & 43.70 & 46.10 & 42.10 \\
Track person through temporal references & 69.07 & 51.55 & 67.47 & 51.05 & 61.86 & 68.67 \\
Conjunction & 78.68 & 44.14 & 57.86 & 59.56 & 45.05 & 71.87 \\
Compound coreference & 71.67 & 35.24 & 60.46 & 66.97 & 56.86 & 65.07 \\
Time reasoning & 46.75 & 34.83 & 47.85 & 34.83 & 45.95 & 37.74 \\
Basic deduction & 57.66 & 54.55 & 47.65 & 45.35 & 57.16 & 49.35 \\
Basic induction & 41.04 & 51.15 & 40.04 & 61.86 & 37.34 & 46.75 \\
Positional reasoning & 51.95 & 51.95 & 51.85 & 51.95 & 51.95 & 51.95 \\
Size reasoning & 46.65 & 46.95 & 53.05 & 46.05 & 45.95 & 49.65 \\
Path finding & 8.41 & 9.21 & 8.21 & 9.21 & 9.21 & 9.41 \\
\midrule
Average & 42.24 & 34.95 & 41.57 & 37.61 & 40.85 & \textbf{44.78} \\
\bottomrule
\end{tabular}
}
\caption{BabiLong accuracy ($\uparrow$) for the 7B MoE models.}
\label{table:babilong}
\end{table}

\subsection{Ablations: Effect of State Size and Forget Gate}
\label{section:state-size}

We ablate the contributions of state size and the forget gate on the recurrent model performance at the 410M scale, comparing a vanilla RNN, an RNN with a forget gate, \method~without a forget gate, and \method~(Table \ref{table:component-ablation}). The matrix-valued \method~model (410M parameters) outperforms the vector-valued RNN (406M parameters) by more than 10 perplexity points on WikiText \cite{wikitext} and 260 points on LAMBADA \cite{paperno2016lambada}, underscoring the critical role of state size. Both components help, and the matrix-valued state accounts for the larger share of the improvement. These results confirm that expanding the state size is the primary driver of performance improvements in recurrent architectures, with gating mechanisms potentially providing further gains.

\begin{table}[h!]
\centering
\scriptsize
\begin{tabular}{cccccr}
\toprule
Forget gate & Matrix-valued state & Wikitext PPL ($\downarrow$) & LAMBADA PPL ($\downarrow$) & State size \\
\midrule
\xmark & \xmark & 33.45 & 302.34 & 1,360 \\
\cmark & \xmark & 26.95 & 89.44 & 1,360 \\
\xmark & \cmark & 26.41 & 48.03 & 86,016 \\
\cmark & \cmark & \textbf{22.97} & \textbf{33.79} & 86,016 \\
\bottomrule
\end{tabular}
\caption{Ablation of the forget gate and the matrix-valued state for 410M models.}
\label{table:component-ablation}
\end{table}

\begin{table}[h!]
\centering
\resizebox{\textwidth}{!}{%
\begin{tabular}{lcccccccc}
\toprule
Attention Pattern & $\bb{W}_q$ & $\bb{W}_k$ & $\bb{W}_v$ & $\bb{W}_g$ & $\bb{W}_f$ & $\bb{W}_o$ & Total ($P$) & State Size ($NKV$) \\
\midrule
Multi-Head & $NK d_{\textrm{m}}$ & $NK d_{\textrm{m}}$ & $NV d_{\textrm{m}}$ & $NV d_{\textrm{m}}$ & $N d_{\textrm{m}}$ & $NV d_{\textrm{m}}$ & $\approx N (2K + 3V) d_{\textrm{m}}$ & $\approx \frac{PKV}{(2K + 3V) d_{\textrm{m}}}$ \\
Multi-Query & $NK d_{\textrm{m}}$ & $K d_{\textrm{m}}$ & $V d_{\textrm{m}}$ & $NV d_{\textrm{m}}$ & $N d_{\textrm{m}}$ & $NV d_{\textrm{m}}$ & $\approx N (K + 2V) d_{\textrm{m}}$ & $\approx \frac{PKV}{(K + 2V) d_{\textrm{m}}}$ \\
Multi-Key & $K d_{\textrm{m}}$ & $NK d_{\textrm{m}}$ & $V d_{\textrm{m}}$ & $NV d_{\textrm{m}}$ & $N d_{\textrm{m}}$ & $NV d_{\textrm{m}}$ & $\approx N (K + 2V) d_{\textrm{m}}$ & $\approx \frac{PKV}{(K + 2V) d_{\textrm{m}}}$ \\
Multi-Value & $K d_{\textrm{m}}$ & $K d_{\textrm{m}}$ & $NV d_{\textrm{m}}$ & $NV d_{\textrm{m}}$ & $N d_{\textrm{m}}$ & $NV d_{\textrm{m}}$ & $\approx 3NV d_{\textrm{m}}$ & $\approx \frac{PK}{3 d_{\textrm{m}}}$ \\
\bottomrule
\end{tabular}
}
\caption{Number of parameters in all linear projections. $N$ denotes the number of heads and $d_{\textrm{m}}$ denotes the model's embedding dimension. For all our models, we use $K = 64$ and $V = 16$ which results in the largest state size for the multi-value attention formulation for constant number of total parameters ($P$).}
\label{table:param-state}
\end{table}

\paragraph{Varying the state size at a fixed parameter count.} To isolate the effect of state size, we vary the key and value head dimensions $K$ and $V$ of a 410M \method~model, scaling the number of heads $N$ accordingly to keep the parameter count constant at ${\approx}410$M. Larger states consistently improve Wikitext perplexity and real-world retrieval accuracy, by up to 7 points on average (Table \ref{table:kv}).

\begin{table}[h!]
\centering
\scriptsize
\begin{tabular}{cccccccccc}
\toprule
 & & \multicolumn{1}{c}{Perplexity ($\downarrow$)} & \multicolumn{7}{c}{Real-world retrieval accuracy ($\uparrow$)} \\
\cmidrule(lr){3-3} \cmidrule(lr){4-10}
$K$ & $V$ & Wikitext & FDA & SWDE & NQ & TriviaQA & SQuAD & DROP & Avg \\
\midrule
16 & 16 & 24.24 & 2.63 & 14.22 & 10.55 & 46.98 & 25.47 & 19.12 & 19.83 \\
16 & 32 & 24.11 & 4.99 & 16.47 & 11.34 & 47.51 & 27.55 & 18.45 & 21.05 \\
32 & 16 & 23.57 & 6.17 & 19.89 & 11.91 & 51.36 & 28.65 & 19.50 & 22.91 \\
32 & 32 & 23.36 & 8.17 & 22.95 & 12.42 & 50.71 & 30.83 & 22.18 & 24.54 \\
64 & 16 & 22.97 & 11.71 & 26.73 & 13.24 & 50.95 & 33.21 & 19.55 & 25.90 \\
64 & 32 & \textbf{22.83} & 12.34 & 29.70 & 13.40 & 50.83 & 34.22 & 21.90 & \textbf{27.07} \\
\bottomrule
\end{tabular}
\caption{Wikitext perplexity ($\downarrow$) and real-world in-context retrieval accuracy ($\uparrow$) of 410M \method~models with different $K, V$.}
\label{table:kv}
\end{table}

\subsection{Training Throughput}
\label{section:throughput}
We measure training throughput for \method, \gdn~\cite{gdn}, \mamba~\cite{mamba2}, and Transformer++ on $8\times$ NVIDIA HGX H100 GPUs connected via NVLink, using our 7B (1B active) MoE configuration. As shown in Figure \ref{fig:throughput}, \mamba~consistently achieves the highest throughput across context lengths, while Transformer++ degrades significantly at longer contexts due to its quadratic complexity.

Although \method~has linear time complexity, its higher constant factors make it more expensive than the linear baselines in the homogeneous setting. However, replacing a single \gdn~layer with \method~layer to the Hybrid \gdn~model (Hybrid \gdn~+ \method-1) achieves throughput comparable to Hybrid \gdn~alone (within $6\%$ at 16k context length) while consistently improving model quality across all benchmarks. We consider this a favorable tradeoff, making \method~layers practical when used sparingly.
Our current implementation is based on the Triton DSL \cite{triton} and is relatively unoptimized, leaving significant room for improvement that could further narrow the remaining throughput gap.

\begin{figure}[h!]
    \centering
    \includegraphics[width=0.5\textwidth]{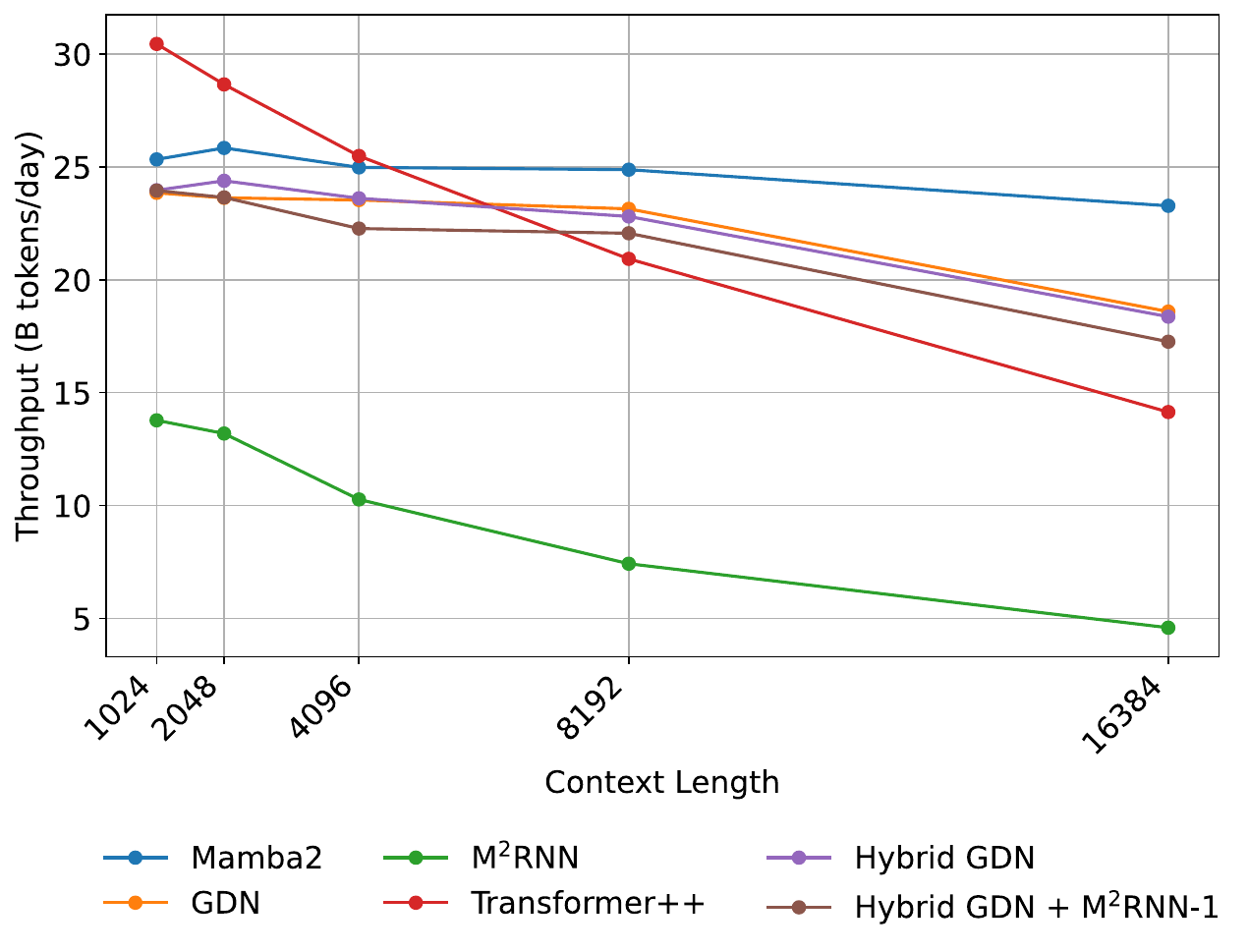}
    \caption{Training throughput measured in billion tokens per day on a machine with 8x NVIDIA HGX H100 GPUs for our 7B MoE configuration.}
    \label{fig:throughput}
\end{figure}

\paragraph{Comparison with FlashRNN.} We benchmark the FlashRNN \cite{flashrnn} sLSTM kernels against our \method~and GRU kernels for homogeneous models at the 7B scale across context lengths (Table \ref{table:flashrnn}). FlashRNN's CUDA-fused and Triton-fused sLSTM kernels run out of memory at 8k and 4k context, respectively, whereas our kernels scale to longer contexts for both GRU and \method. At 2k context, \method~already outperforms FlashRNN's CUDA-fused sLSTM (13.19 vs.\ 12.14) and comes close to its Triton-fused kernel. \method~performs more FLOPs per recurrent step than vector-valued RNNs, so our GRU kernel is faster than \method~at long contexts (this is the cost of the larger state). Throughput alone does not tell the full story: sLSTM and GRU trail Transformers in perplexity by up to 4 points \cite{flashrnn}, whereas \method~is competitive with linear RNNs (Table \ref{table:7b}).

\begin{table}[h!]
\centering
\scriptsize
\begin{tabular}{lccccc}
\toprule
Model & 1024 & 2048 & 4096 & 8192 & 16384 \\
\midrule
\method & 13.77 & 13.19 & 10.27 & 7.42 & 4.59 \\
GRU & 12.81 & 12.62 & 12.44 & 10.12 & 7.63 \\
sLSTM (CUDA) & 7.25 & 5.57 & 3.19 & 1.68 & 0.79 \\
sLSTM (CUDA fused) & 15.37 & 12.14 & 7.55 & OOM & OOM \\
sLSTM (Triton fused) & 21.29 & 14.65 & OOM & OOM & OOM \\
\bottomrule
\end{tabular}
\caption{Training throughput ($\uparrow$) of homogeneous 7B models across context lengths. sLSTM uses the FlashRNN kernels; \method~and GRU use our kernels. OOM denotes out of memory.}
\label{table:flashrnn}
\end{table}

\paragraph{Inference.} Table \ref{table:inference} compares the inference throughput and memory of \method~and Transformer++ as the batch size grows, with a fixed 1,024-token prompt and 4,096 decoded tokens. Because \method~maintains a constant-size state instead of a growing KV cache, its advantage grows with the batch size, reaching $2.67\times$ higher throughput with $0.78\times$ the memory at batch size 32.

\begin{table}[h!]
\centering
\scriptsize
\resizebox{\linewidth}{!}{%
\begin{tabular}{rcccccc}
\toprule
Batch size & Transformer++ tok/s & Transformer++ mem (GB) & \method~tok/s & \method~mem (GB) & Speedup & Memory ratio \\
\midrule
1 & 371.58 & 5.19 & \textbf{538.99} & \textbf{4.70} & 1.45$\times$ & 0.90$\times$ \\
4 & 910.26 & 12.11 & \textbf{1416.98} & \textbf{10.09} & 1.56$\times$ & 0.83$\times$ \\
16 & 1569.26 & 39.50 & \textbf{3586.21} & \textbf{31.04} & 2.28$\times$ & 0.79$\times$ \\
32 & 1815.06 & 77.36 & \textbf{4847.33} & \textbf{60.28} & 2.67$\times$ & 0.78$\times$ \\
\bottomrule
\end{tabular}
}
\caption{Inference throughput and memory with a 1,024-token prompt and 4,096 decoded tokens.}
\label{table:inference}
\end{table}

\paragraph{Context parallelism.}
\label{section:context-parallel}
The scan-based context parallelism used for \mamba~\cite{mamba2} and \gdn~\cite{gdn} relies on their transition being associative, which lets chunk-final states be composed across GPUs with a distributed prefix computation. The non-linear transition of \method~has no associative combination rule, so this approach does not apply. A pipeline-style scheme is feasible instead: each GPU processes its local chunk, and for the \method~layers only, chunk $i+1$ waits on a single $K \times V$ state handoff from chunk $i$ (once per chunk boundary, not per token). This creates a standard pipeline bubble that can be hidden with microbatching. Since \method~layers are only 1--5 of the 40 layers in our hybrids, the remaining layers use ordinary context parallelism (ring attention or \mamba~context parallelism), so the bubble affects a small fraction of the model. Microbatching is what fills the pipeline, and the available headroom shrinks as the batch size shrinks. We therefore measure the extreme case with a single microbatch, where the bubble cannot be hidden at all (Table \ref{table:context-parallel}). Even with this naive implementation, Hybrid \mamba~+ \method-1 scales to 128k tokens at a context-parallel world size of 8 and is $2.6\times$ faster than Transformer++ at 128k; a properly microbatched implementation should improve throughput further. Separately, since \method~reaches its larger state through a large number of heads, tensor parallelism across heads can reduce per-device memory, although scaling to very large tensor-parallel domains incurs a significant communication cost.

\begin{table}[h!]
\centering
\scriptsize
\begin{tabular}{rcccc}
\toprule
Context length & CP world size & Transformer++ & Hybrid \mamba~+ \method-1 & Hybrid \mamba \\
\midrule
4k & 1 & 27.76 & 25.13 & 25.46 \\
8k & 1 & 25.23 & 23.88 & 25.17 \\
16k & 1 & 20.43 & 22.52 & 24.49 \\
32k & 2 & 10.28 & 15.93 & 17.95 \\
64k & 4 & 5.41 & 11.94 & 15.14 \\
128k & 8 & 2.82 & 7.32 & 9.94 \\
\bottomrule
\end{tabular}
\caption{Training throughput ($\uparrow$) with context parallelism (CP) using a single microbatch. \method~layers use a pipeline-style state handoff; attention and \mamba~layers use ring attention and \mamba~context parallelism, respectively.}
\label{table:context-parallel}
\end{table}

\section{Limitations and Future Work}
\method's sequential non-linear recurrence yields higher constants than linear alternatives, a hand-tuned CUTLASS implementation or scan-style approximations such as ParaRNN \cite{danieli2025pararnn} could narrow this gap. \method~layers are significantly more expensive than linear RNNs limiting their practical usage to only a few layers in the entire model (we demonstrate that this setting works well and even outperforms 2-way hybrid models). Our experiments span 410M dense and 7B (1B active) MoE models trained on 100B tokens at 4k context; validating \method~at larger parameter scales requires significantly more compute. We hypothesize that \method~layers are also applicable to other modalities like audio, code, math etc but we haven't validated this empirically.

\section{Conclusion}
We introduce \Method~(\method), a non-linear RNN architecture with matrix-valued hidden states that addresses key limitations of existing linear recurrent models. By combining expressive non-linear state transitions with a forget gate and an outer product state expansion, \method~achieves strong state-tracking capabilities.

Our experiments demonstrate that, in the homogeneous setting, \method~is competitive with \mamba~\cite{mamba2} and \gdn~\cite{gdn} at both the 410M dense and 7B (1B active) MoE scales. More importantly, hybrid models combining \method~with attention layers significantly outperform equivalent \mamba~and \gdn~hybrids in language modeling, in-context retrieval, and long-context generalization. We further show that even a single \method~layer yields substantial accuracy gains with only $6\%$ throughput degradation, making it a practical enhancement for existing architectures.

The matrix-valued recurrence in \method~avoids the batch padding that vector-valued recurrences need under the FlashRNN algorithm \cite{flashrnn}, enabling tensor core utilization independent of batch size. Combined with our custom kernels, this makes \method~a viable component for training state-of-the-art production language models.

We also present two complementary tensor parallelism strategies for \method. The topology-aware approach adopts a grouped-value formulation that requires no additional communication beyond standard TP, while the topology-independent approach preserves the parameter count regardless of TP world size at the cost of a small number of extra AllReduce operations. Together, these strategies make \method~readily deployable across diverse multi-GPU configurations.

\section*{Acknowledgement}
We would like to thank Songlin Yang, Bharat Runwal and Kevin Li for providing valuable feedback and discussion throughout the duration of this project.
We gratefully acknowledge the support of the Schmidt Sciences AI2050 fellowship, the Google ML and Systems Junior Faculty Awards, and the Google Research Scholar program.

\bibliographystyle{plainnat}
\bibliography{paper}

\appendix
\section{Representation Power of \method~Layers}
\label{appendix:proof}
Proposition \ref{theorem}: The \method~recurrence can represent all tasks representable by non-linear vector-valued RNNs and hence can represent regular languages.

\begin{proof}
    We show that \method~can simulate a vector-valued nonlinear RNN, 
    which is known to recognize all regular languages \cite{merrill2019sequential}.

    Consider the parameter setting
    \[
        f_t = 0, 
        \qquad 
        w_r = \mathbf{0}_V,
        \qquad
        q_t = k_t = e_1 \in \mathbb{R}^K,
    \]
    where $e_1 = [1, 0, \dots, 0]^\top$ is the first standard basis vector.

    Under this choice, the matrix-valued recurrence reduces in its first row to
    \begin{align}
        \mathbf{Z}_t 
        &= \tanh\!\left(
            \mathbf{H}_{t-1} \mathbf{W}
            +
            \begin{bmatrix}
                v_t^\top \\
                \mathbf{0}_V^\top \\
                \vdots \\
                \mathbf{0}_V^\top
            \end{bmatrix}
        \right),
    \end{align}
    and the output is obtained from the first row,
    \[
        y_t = (\mathbf{Z}_t)_{1,:}.
    \]

    This is precisely a standard vector-valued nonlinear RNN with $\tanh$ activation applied to the first row of the state matrix. Since vector-valued nonlinear RNNs can represent any regular language \cite{merrill2019sequential}, it follows that \method~can represent any regular language.
\end{proof}

\section{RMSNormTP Module}
\label{sec:rmsnormtp}
The forward computation for RMSNorm \cite{rmsnorm} applied on the input activations $x$ is given by:
\begin{align}
    s &= \frac{1}{\sqrt{\frac{1}{d}\sum_{j=1}^{d} x_j^2}} \\
    y_i &= (w_i \odot x_i) s
\end{align}
where $i$ indicates the $i^{th}$ feature of the input or output, $d$ is the number of features and $w$ is some learnable weight for the RMSNorm module. We cache $s$ during the forward computation in HBM for backward computation. Similarly, the backward computation is given by:
\begin{align}
    p_i &= w_i \odot \nabla_{y_i} \mathcal{L}\\
    r &= \sum_{j=1}^{d}p_j x_j\\
    \nabla_{x_i} \mathcal{L} &= s p_i - \frac{rs^3x_i}{d} \\
    \nabla_{w_i} \mathcal{L} &= (\nabla_{y_i} \mathcal{L} \odot x_i)s
\end{align}

To use RMSNorm \cite{rmsnorm} module with \method~layers with Tensor Parallelism (TP), we shard the weight $w$ for RMSNorm among the GPUs participating in TP. Since the incoming activations $x$ are also sharded along the feature dimension, to compute the inverse of the normalization term $s$, we require an AllReduce operation in the forward. We propose the RMSNormTP module which modifies the RMSNorm module to compute $s$ in forward and $r$ in backward as:
\begin{align}
    \label{eqn:s-all}
    s &= \frac{1}{\sqrt{\frac{1}{d} {\color{blue}\sum_{n=1}^{N_\textrm{TP}}} \sum_{j \in D_n} x_j^2}} \\
    \label{eqn:r-all}
    r &= {\color{blue}{\sum_{n=1}^{N_\textrm{TP}}}} \sum_{j \in D_n}p_j x_j
\end{align}
where $D_n$ indicates the sharded features on $n^{th}$ GPU for TP. The summation in \blue{blue colour} indicates an AllReduce operation across GPUs in Equations \ref{eqn:s-all}, \ref{eqn:r-all}. Note that for both AllReduce operations, we only need to communicate a single element.

\section{Additional Long-Context Results}
\label{appendix:additional}

\paragraph{Retrieval with vector-valued non-linear RNNs.} On the RULER needle-in-a-haystack tasks (Section \ref{section:ruler}), vector-valued RNN and GRU models achieve close to $0\%$ accuracy. This suggests that a large state is needed for retrieval in non-linear RNNs, but we do not have a mechanistic explanation of how key-value associations written with $k_t v_t^\top$ survive the non-linear state update, and the retrieval accuracy of homogeneous \method~remains below that of linear RNNs with larger states such as \gdn~(Table \ref{table:7b-real-retrieval}). In hybrids, attention layers take over exact retrieval (Section \ref{section:in-context-retrieval-real}).

\end{document}